\pdfoutput=1
\documentclass{article}

\usepackage[main, final]{neurips_2026}

\usepackage[utf8]{inputenc}
\usepackage[T1]{fontenc}
\usepackage{amsmath}
\usepackage{amssymb}
\usepackage{amsfonts}
\usepackage{amsthm}
\usepackage{bm}
\usepackage{booktabs}
\usepackage{float}
\usepackage{graphicx}
\usepackage{wrapfig}
\usepackage{multirow}
\usepackage{subcaption}
\usepackage{algorithm}
\usepackage{algorithmic}
\usepackage{nicefrac}
\usepackage{microtype}
\usepackage{xcolor}
\usepackage{url}
\usepackage{hyperref}
\usepackage{authblk}

\usepackage{dsfont}
\usepackage{tikz}
\usetikzlibrary{matrix,positioning,arrows.meta,fit,calc,shapes.geometric}

\definecolor{mydarkred}{rgb}{0.6,0,0}
\definecolor{mydarkgreen}{rgb}{0,0.6,0}
\hypersetup{
    colorlinks=true,
    linkcolor=mydarkred,
    citecolor=mydarkgreen
}

\usepackage[capitalize,noabbrev]{cleveref}

\theoremstyle{plain}
\newtheorem{theorem}{Theorem}
\newtheorem{lemma}[theorem]{Lemma}
\newtheorem{proposition}[theorem]{Proposition}

\theoremstyle{definition}

\newtheorem{assumption}[theorem]{Assumption}
\theoremstyle{remark}

\crefname{assumption}{Assumption}{Assumptions}

\newcommand{\method}{Joint Envelope Conformal Selection}
\newcommand{\abbr}{JECS}
\newcommand{\baselinefull}{Joint Max Conformal Selection}
\newcommand{\baseline}{JMCS}
\newcommand{\metricfull}{global contamination rate}
\newcommand{\metricabbr}{GCR}
\newcommand{\fdpfull}{global contamination proportion}
\newcommand{\fdpabbr}{GCP}
\newcommand{\metricop}{\ensuremath{\mathrm{GCR}}}
\newcommand{\fdpop}{\ensuremath{\mathrm{GCP}}}

\newcommand{\cD}{\mathcal{D}}

\newcommand{\cS}{\mathcal{S}}
\newcommand{\cP}{\mathcal{P}}
\newcommand{\cX}{\mathcal{X}}
\newcommand{\cI}{\mathcal{I}}
\newcommand{\cN}{\mathcal{N}}
\newcommand{\E}{\mathbb{E}}
\newcommand{\Pp}{\mathbb{P}}
\newcommand{\ind}{\mathbf{1}}

\newcommand{\viol}[1]{#1$^{\dagger}$}

\title{Provable Joint Decontamination for Benchmarking Multiple Large Language Models}

\author[1,2,3]{Zhenlong Liu}
\author[4]{Hao Zeng}
\author[1,*]{Hongxin Wei}

\affil[1]{Department of Statistics and Data Science, Southern University of Science and Technology}
\affil[2]{Tencent HY}
\affil[3]{Shanghai Innovation Institute}
\affil[4]{School of Mathematical Sciences, University of Electronic Science and Technology of China}

\begin{document}
{\renewcommand{\thefootnote}{*}\footnotetext{Correspondence to: Hongxin Wei \texttt{<weihx@sustech.edu.cn>}.}}

\maketitle

\begin{abstract}
Benchmark data contamination has become a central challenge in LLM evaluation: when evaluation examples appear in the training data of one or more audited models, reported performance can be inflated and cross-model comparisons become unreliable.
A broad line of training-data detection work designs scores to quantify how strongly a model memorizes a given data point, but these score-based methods lack theoretical guarantees. Recent conformal approaches provide provable false-identification control for a \emph{single model}; however, applying them separately to each model can produce model-specific benchmarks, undermining fair comparison across models. 
In this work, we formalize multi-model benchmark decontamination as a joint selection problem and propose \method{} (\textbf{\abbr{}}), a conformal procedure that enables \metricfull{} (\metricabbr{}) control under stated assumptions.
Specifically, \abbr{} computes per-model conformal $p$-values, aggregates them by the per-item maximum, and reconstructs a conservative envelope of the max-$p$ null distribution from right-tail observations above a data-driven threshold.
By applying the adaptive Benjamini--Hochberg (BH) procedure to the envelope-rescaled values, we select a benchmark with provable \metricabbr{} control.
Extensive experiments across various models and benchmarks demonstrate that \abbr{} achieves higher power than the max-$p$ baseline while consistently maintaining the target \metricabbr{} control.
Our code is available at \url{https://github.com/ml-stat-Sustech/Joint-Pure-Benchmark-Selection}.
\end{abstract}

\section{Introduction}
\label{sec:intro}

The remarkable achievements of large language models (LLMs) have been driven in large part by the scale and diversity of their pretraining corpora~\citep{kaplan2020scaling,chang2024survey}. 
However, the construction of large pretraining corpora can lead to overlap between training data and evaluation benchmarks, a phenomenon commonly referred to as benchmark data contamination~\citep{sainz2023nlp, deng2024investigating, xu2024benchmark,balloccu2024leak}. Such overlap can overstate reported performance, obscure whether models truly generalize beyond their training data, and weaken the credibility of benchmark-based evaluation.
This highlights the need to construct a decontaminated benchmark for multiple audited models.

To address this issue, a broad line of work~\citep{carlini2021extracting, zhang2025mink,li2024membership} designs detection scores, such as Min-K\%~\citep{shi2024detecting}, to quantify how strongly a model memorizes a given data point. 
Motivated by the lack of theoretical guarantees in these score-based methods, recent work~\citep{liu2026provable} leverages conformal inference to identify training data for a \emph{single model} with provable false-identification control. 
However, applying such a procedure separately to each model can produce model-specific benchmarks, undermining fair comparison across models. 
In practical benchmark evaluation, auditors require a shared decontaminated benchmark that supports fair assessment across multiple models. 
This motivates us to develop methods for selecting \emph{one} shared benchmark with a controlled contamination rate across all audited models.

In this paper, we formulate joint benchmark decontamination as a joint selection problem, where the goal is to select a single benchmark that is decontaminated with respect to multiple audited models. Specifically, we call a sample \emph{jointly pure} if it does not appear in the training data of any audited model, and define the \metricfull{}~(\metricabbr{}) as the expected fraction of contaminated samples in the selected set.
To control the \metricabbr{}, we propose \method{} (\textbf{\abbr{}}), which first obtains a valid $p$-value under the joint null for each candidate by taking the maximum of its per-model conformal $p$-values.   
To recover power, we further reconstruct a conservative envelope function for the cumulative distribution function (CDF) of the max-$p$ null distribution, map each max-$p$ value through the fitted function, and then apply the adaptive BH procedure~\citep{benjamini1995controlling, benjamini2001control}.
We establish theoretical guarantees showing that \abbr{} controls the contamination rate at a user-specified level $\alpha$.

Extensive experiments on simulated data and real LLM benchmarks demonstrate the effectiveness of \abbr{} for joint benchmark decontamination. Across all settings, \abbr{} controls the \metricfull{} at the user-specified target, while neither the union nor the intersection of per-model conformal selections does so. For instance, on the synthetic setup at $\alpha=0.1$, \abbr{} attains \metricabbr{}~$=0.038$, whereas the union and intersection rules violate the target with \metricabbr{}~$=0.763$ and $0.366$, respectively. Furthermore, our method improves power over the max-$p$ baseline. For example, on ArXivTection with Pythia-6.9B~\citep{biderman2023pythia} at $\alpha=0.1$ using the Min-K\%++ score~\citep{zhang2025mink}, our method improves power from $0.094$ to $0.447$.
Overall, these results show that \abbr{} preserves the contamination-rate control required for joint selection while recovering much of the power lost by the conservative max-$p$ baseline.

We summarize our contributions as follows:
\begin{enumerate}
    \item  We formulate multi-model benchmark decontamination as a joint selection problem, where a candidate item is jointly pure only if it is absent from the training data of every audited model. We introduce the \metricfull{}~(\metricabbr{}) as the contamination-rate criterion for selecting one shared benchmark across multiple audited models.
    \item We propose \method{} (\textbf{\abbr{}}), a joint selection procedure for constructing shared decontaminated benchmarks across multiple audited models.  
    Specifically, \abbr{} mitigates the super-uniform conservatism of the valid max-$p$ baseline, thereby retaining \metricfull{} control while substantially improving power. 
    
    \item We establish theoretical guarantees showing that \abbr{} controls \metricabbr{} under stated conditions. We further validate these guarantees through extensive experiments, demonstrating the effectiveness of \abbr{} on both simulated data and real-world LLM benchmarks.
\end{enumerate}

\section{Preliminary}
\label{sec:background}
\paragraph{Setup.}
Let $\theta$ be a target language model trained on a private corpus $\cD_{\mathrm{train}}(\theta)$, and let $\cX$ denote the input space of token sequences. For a sample $x\in\cX$, the binary membership indicator
\begin{equation}
\label{eq:membership}
    M(x;\theta)=\mathds{1} \{x\in\cD_{\mathrm{train}}(\theta)\}\in\{0,1\}
\end{equation}
records whether $x$ was used to train $\theta$. Specifically, $M=1$ indicates a member and $M=0$ a non-member. Throughout the paper, we audit $K$ such models $\{\theta_k\}_{k=1}^K$ on a candidate pool $\{x_i\}_{i=1}^n$, and write $M_i^k=M(x_i;\theta_k)$ for the membership status of candidate instance $i$ under model $\theta_k$.

\paragraph{Training data detection.}
Training data detection, also termed membership inference attack (MIA), aims to infer the membership indicator \(M(x;\theta)\) from query access to the model \(p_\theta\). The standard pipeline assigns each sample a detection score $T(x;\theta)$ and predicts membership by a level-set rule, with the inequality direction determined by the score convention. In this paper, we orient $T$ so that smaller values provide stronger evidence of non-membership, and hence a threshold rule may be written as $\widehat{M}=\mathds{1}\{T(x;\theta)\ge\tau\}$ with a threshold $\tau$ determined by a validation set~\citep{shokri2017membership,ye2022enhanced}. A representative score is Min-K\%++~\citep{zhang2025mink}, which for a sequence $x=(t_1,\ldots,t_L)$ averages a normalized log-probability over the bottom-$K\%$ tokens:
\begin{equation}
\label{eq:minkpp}
    T_{\textsc{Min-K\%++}}(x;\theta)
    =
    \frac{1}{|\cI|}\sum_{\ell\in\cI}
    \frac{\log p_\theta(t_\ell\mid t_{<\ell})-\mu_\ell}{\sigma_\ell},
    \qquad
    \mu_\ell=\E_{t\sim p_\theta(\cdot\mid t_{<\ell})}[\log p_\theta(t\mid t_{<\ell})],
\end{equation}
where $\sigma_\ell$ is the standard deviation of $\log p_\theta(\cdot\mid t_{<\ell})$ under the next-token distribution and $\cI\subset\{1,\ldots,L\}$ indexes the lowest $K\%$ tokens by per-token log-probability. Other widely used scores include perplexity~\citep{carlini2021extracting}, Min-K\%~\citep{shi2024detecting}, and Modified Entropy~\citep{song2021systematic}. 
These scores provide useful per-instance evidence of memorization by a specific model, but they do not provide theoretical guarantees for the resulting predictions.

\paragraph{Conformal inference for training data identification.} To provide provable evidence, recent work~\citep{liu2026provable} formulates training-data identification as a multiple-hypothesis testing problem. For each audited model $\theta_k$, the model-wise null and alternative are:
\begin{equation}
    H_{0,i}^{k}: x_i\in\cD_{\mathrm{train}}(\theta_k),
    \qquad
    H_{1,i}^{k}: x_i\notin\cD_{\mathrm{train}}(\theta_k).
\end{equation}
Rejecting $H_{0,i}^{k}$ means that $x_i$ is identified as pure with respect to $\theta_k$; a false identification occurs when a contaminated item is incorrectly selected.\footnote{The testing direction here is reversed from the convention in~\citet{liu2026provable}, where the null is non-member. We reverse the null because benchmark construction selects non-members.} The conformal $p$-value~\citep{VovkNG03,vovk2005algorithmic} calibrates the detection scores into valid $p$-values:
\begin{equation}
\label{eq:prelim-model-p}
    p_i^k
    =
    \frac{1+\sum_{x_\ell \in\cD_{\mathrm{cal}}}\ind\{T(x_\ell;\theta_k)\le T(x_i;\theta_k)\}}
    {|\cD_{\mathrm{cal}}|+1},
\end{equation}
where $\cD_{\mathrm{cal}}$ is a calibration set of known members of $\theta_k$ drawn from the same covariate distribution as contaminated candidates. The Benjamini--Hochberg (BH) procedure~\citep{benjamini1995controlling} is then used to select a subset:
\begin{equation}
\label{eq:single-model-selection}
    \cS_k
    =
    \left\{i:p_i^k\le\frac{\alpha\,r^{*}}{n}\right\},
    \text{ where }
    r^{*}
    =
    \max\!\left\{r\in\{1,\ldots,n\}:p_{(r)}^k\le\frac{\alpha\,r}{n}\right\},
\end{equation}
where $p_{(1)}^k\le\cdots\le p_{(n)}^k$ are the sorted model-wise $p$-values. In the original training-data-identification setting, prior works~\citep{bates2023testing,liu2026provable} have shown that the expected fraction of $\theta_k$-non-members among the items identified as members can be controlled at a user-specified level $\alpha \in (0, 1)$. However, this guarantee is model-specific: applying the procedure separately to different audited models yields different selected subsets $\cS_1,\ldots,\cS_K$. In practical benchmark evaluation, auditors need a single decontaminated benchmark for a fair comparison across all audited models. This gap motivates the joint benchmark decontamination problem below.

\section{Joint Benchmark Decontamination}
\label{sec:method}

\subsection{Problem formulation}
\label{sec:problem}

We formulate the \textbf{joint benchmark decontamination} problem, whose goal is to select one shared benchmark that is decontaminated with respect to all audited models. Unlike the model-wise setting above, contamination in this problem is defined globally across the audited model collection. An instance $x_i$ is \emph{contaminated} if it appears in the training data of at least one audited model, and is \emph{jointly pure} only if it appears in none. Accordingly, the joint null and alternative hypotheses are:
\begin{equation}
\label{eq:global-null}
    H_{0,i}:\; \exists\, k^{*}\in\{1,\ldots,K\},\; x_i\in\cD_{\mathrm{train}}(\theta_{k^{*}}),
    \qquad
    H_{1,i}:\; \forall\, k,\; x_i\notin\cD_{\mathrm{train}}(\theta_k).
\end{equation}

Rejecting $H_{0,i}$ indicates that $x_i$ is inferred to be jointly pure and is therefore selected for inclusion in the benchmark. We therefore quantify error by the contamination rate within the selected benchmark. Formally, for a selected set $\cS\subseteq\{1,\ldots,n\}$, we define the \fdpfull{} (\fdpabbr{}) and \metricfull{} (\textbf{\metricabbr{}}) as follows:
\begin{equation}
\label{eq:scp}
    \fdpop(\cS)
    = 
    \frac{\sum_{i=1}^n\mathds{1}\{i\in\cS,\;H_{0,i}\text{ is true}\}}
    {1\vee|\cS|},
    \qquad
    \metricop(\cS)=\E[\fdpop(\cS)],
\end{equation}
 where we denote $a\vee b= \max\{a,b\}$. Our goal is to return a subset $\cS$ such that $\metricop(\cS)\le\alpha$, where $\alpha\in(0,1)$ is a user-specified level. Beyond GCR control, we also seek a procedure that selects as many true candidates as possible. To quantify this objective, we define the \textbf{power} as
\begin{equation}
    \mathrm{Power}(\cS)= \E \left[ \frac{ \sum_{i=1}^n \mathds{1}\{i\in \cS, H_{1,i} \text{ is true}\}}{ 1 \vee \sum_{i=1}^n \mathds{1}\{H_{1,i} \text{ is true}\} } \right].
\end{equation}
  An ideal method should control the GCR at the prescribed level $\alpha$ with power as high as possible.

\paragraph{Naive per-model composition fails to control \metricabbr{}.}
\begin{wraptable}{r}{0.52\linewidth}
\vspace{-10pt}
\centering
\caption{Realized \fdpabbr{} of naive per-model composition on synthetic data. 
Entries marked with a $\dagger$ exceed $\alpha$, indicating failure to meet the \metricabbr{} target.
}
\label{tab:motivation-union-intersection}
\vspace{1pt}
\footnotesize
\setlength{\tabcolsep}{3.5pt}
\begin{tabular}{lccc}
\toprule
Procedure & $\alpha=0.1$ & $\alpha=0.2$ & $\alpha=0.3$ \\
\midrule
Union $\cup_k\cS_k$ & \viol{0.763} & \viol{0.764} & \viol{0.764} \\
Intersection $\cap_k\cS_k$ & \viol{0.366} & \viol{0.614} & \viol{0.757} \\
\abbr{} & 0.038 & 0.087 & 0.162 \\
\bottomrule
\end{tabular}
\vspace{-5pt}
\end{wraptable}
A natural but invalid approach is to first form the model-wise selections $\cS_1,\ldots,\cS_K$ from the conformal $p$-values in \Cref{eq:prelim-model-p} using the single-model rule in \Cref{eq:single-model-selection}, and then combine them by union or intersection. The union $\cup_k\cS_k$ is too liberal: an item may be selected because it appears pure for one model while being contaminated for another. The intersection $\cap_k\cS_k$ is more conservative, but still lacks \metricabbr{} control because the denominator changes after composition and the model-wise selections are dependent. As shown in Table~\ref{tab:motivation-union-intersection}, both composition violate the target on a controlled synthetic setup with $K=4$ audited models and known joint-purity labels (full configuration in \Cref{app:synthetic-motivation}): Union stays near $\fdpop\approx0.76$, and Intersection increasingly exceeds $\alpha$ as $\alpha$ grows. Thus, joint \metricabbr{} control for \Cref{eq:global-null} requires testing $H_{0,i}$ directly through an aggregated statistic rather than composing per-model decisions.
 
\subsection{A max-$p$ baseline: Joint Max Conformal Selection (JMCS)}
\label{sec:scaffold}

We now adapt the model-wise conformal $p$-values from \Cref{sec:background} to the joint selection problem.
Throughout this section, we use a shared calibration set $\cD_{\mathrm{cal}}=\{x_\ell\}_{\ell=n+1}^{n+m}$ whose items are known members of every audited model, i.e., $x_\ell\in\cD_{\mathrm{train}}(\theta_k)$ for all $k\in \{1, \dots, K\}$, and are drawn from the same candidate-pool covariate distribution. In practice, constructing $\cD_{\mathrm{cal}}$ only requires high-confidence shared members rather than each model's full training corpus. As discussed in \Cref{app:calibration-feasibility}, such items can be collected from public benchmark training splits routinely mixed into LLM training, provider-disclosed pretraining sources such as Common Crawl, Wikipedia, arXiv, and GitHub, and regulatory training-data disclosures, then restricted to the benchmark domain and matched on score-relevant properties including topic, programming language, and sequence length.

Applying the conformal construction in \Cref{eq:prelim-model-p} to each audited model yields model-wise $p$-values $\{p_i^k\}_{k=1}^K$.
By exchangeability, for every model-wise null $H_{0,i}^k$,
\begin{equation}
\label{eq:model-validity}
    \Pp(p_i^k\le t\mid H_{0,i}^k)\le t,\qquad t\in[0,1].
\end{equation}

To test the joint null $H_{0,i}$ via one statistic, we aggregate the $K$ model-wise conformal $p$-values by their maximum:
\begin{equation}
\label{eq:max-p}
    p_i^*=\max_{1\le k\le K}p_i^k.
\end{equation}
The aggregated value $p_i^*$ is small only if every model's $p$-value is small, which is the operational meaning of joint purity. The following lemma states that $p_i^*$ is itself a valid $p$-value for the joint null.

\begin{lemma}[Validity of the max-$p$]
\label{lem:max-valid}
If \cref{eq:model-validity} holds for each model-wise null, then $p_i^*$ is valid for the joint null:
\begin{equation}
    \Pp(p_i^*\le t\mid H_{0,i})\le t,\qquad t\in[0,1].
\end{equation}
\end{lemma}
\begin{proof}
Under $H_{0,i}$, there exists $k^*$ such that $H_{0,i}^{k^*}$ is true. Since $p_i^*=\max_{k}p_i^k\ge p_i^{k^*}$, we have $\Pp(p_i^*\le t\mid H_{0,i})\le\Pp(p_i^{k^*}\le t\mid H_{0,i}^{k^*})\le t$.
\end{proof}

\cref{lem:max-valid} immediately delivers a baseline procedure: apply the Benjamini--Hochberg (BH)~\citep{benjamini1995controlling} step-up rule to the maxima $\{p_i^*\}_{i=1}^n$ at level $\alpha$. Letting $p_{(1)}^*\le\cdots\le p_{(n)}^*$ be the sorted maxima, the rule selects
\begin{equation}
\label{eq:bh-baseline}
    \cS
    =
    \left\{i:p_i^*\le\frac{\alpha\,r^{*}}{n}\right\},
    \text{ where }
    r^{*}
    =
    \max\!\left\{r\in\{1,\ldots,n\}:p_{(r)}^*\le\frac{\alpha\,r}{n}\right\},
\end{equation}
 We refer to this procedure as the \baselinefull{} (\baseline{}). \cref{lem:max-valid} together with the standard BH analysis~\citep{benjamini1995controlling} yields finite-sample \metricabbr{} control under the joint null. \baseline{} is the natural minimum-viable procedure for joint pure benchmark selection, and to our knowledge no prior work has analyzed it as such; we use it as the baseline against which \abbr{} is compared in \cref{sec:experiments}, and summarize the full procedure in \cref{alg:jmcs} of \cref{app:jmcs}.

\paragraph{The super-uniformity tax.}
\begin{wrapfigure}{r}{0.5\linewidth}
\vspace{-0.8em}
\centering
\includegraphics[width=0.95\linewidth]{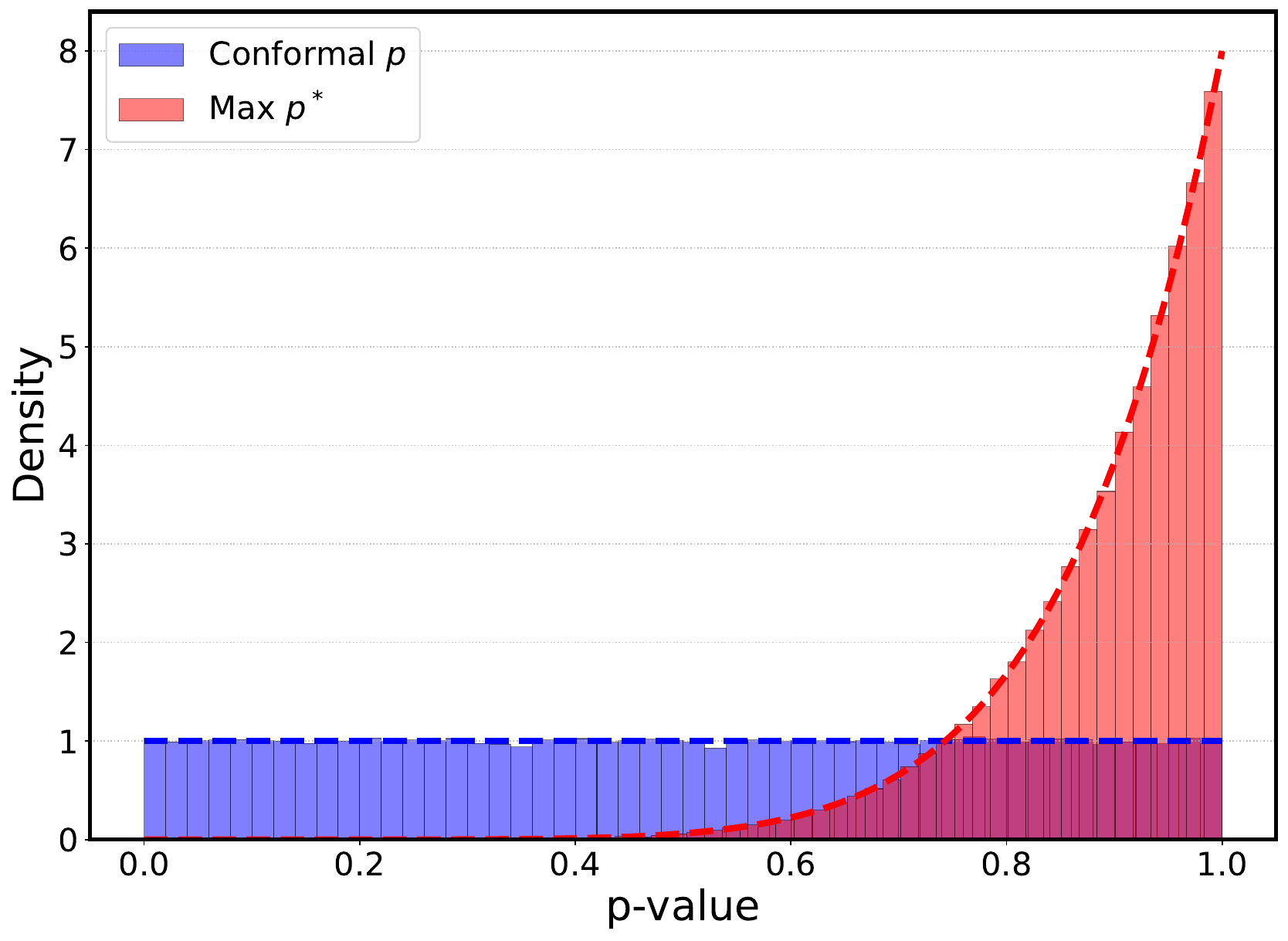}
\caption{Null density of the per-model conformal $p$ versus $p_i^*=\max_k p_i^k$ at $K=8$.}
\label{fig:super-uniformity-tax}
\vspace{-1.0em}
\end{wrapfigure}
\baseline{} is valid in finite samples but substantially conservative. The BH cutoff $\alpha r/n$ is calibrated to the uniform scale, whereas under $H_{0,i}$ the maximum statistic $p_i^*=\max_{k}p_i^k$ is strictly super-uniform: in the stylized homogeneous null with $p_i^1,\ldots,p_i^K\stackrel{\mathrm{iid}}{\sim}\mathrm{Unif}(0,1)$, the null CDF is $F_0(t)=t^K$ and the density $f_0(t)=Kt^{K-1}$ vanishes at $t=0$ and concentrates near $t=1$ (\cref{fig:super-uniformity-tax}, $K=8$). Almost all null mass therefore sits above the BH cutoff: a rejection requires $p_i^*\le\alpha r/n$, far below the natural null scale $(r/n)^{1/K}$, and the gap widens as $\alpha$ tightens or $K$ grows. Although \cref{lem:max-valid} guarantees \metricabbr{} control, this procedure can lose substantial power as $K$ increases.

\subsection{\method{} (\abbr{})}
\label{sec:jcs-method}

As discussed above, $p_i^*$ is expected to follow a monotone increasing distribution on $[0,1]$ under the null, resembling a $\mathrm{Beta}$ distribution and concentrating near 1. While this preserves validity, it often leads to overly conservative selection. A natural remedy is therefore to estimate the null CDF of $p_i^*$ and use the estimated CDF to recalibrate $p_i^*$, improving selection power while retaining validity.

\begin{algorithm}[t]
\caption{\method{} (\abbr{})}
\label{alg:hatg}
\begin{algorithmic}[1]
\REQUIRE Candidate items $\{x_i\}_{i=1}^n$, audited models $\{\theta_k\}_{k=1}^K$, calibration data $\cD_{\mathrm{cal}}=\{x_i\}_{i=n+1}^{n+m}$, target level $\alpha$, tail threshold $\lambda$.
\FOR{$k=1,\ldots,K$}
    \STATE Construct conformal $p$-values $\{p_i^k\}_{i=1}^n$ via \cref{eq:prelim-model-p}.
\ENDFOR
\STATE Aggregate $p_i^*=\max_k p_i^k$ for $i=1,\ldots,n$.
\STATE Estimate $\widehat g(\lambda^+)$ and $\widehat G_n$ from the right tail $\{p_i^*:p_i^*>\lambda\}$.
\STATE Form the envelope $\widehat F_{\mathrm{fit}}$ via \cref{eq:anchor} and \cref{eq:fit-envelope}.
\STATE Transform $\widetilde p_i=\widehat F_{\mathrm{fit}}(p_i^*)$ and estimate $\widehat\pi_0$ via \cref{eq:pi0}.
\STATE Return the selected set $\cS$ via \cref{eq:final-set}.
\end{algorithmic}
\end{algorithm}

Let $F_0$ and $F_1$ denote the CDFs of $p_i^*$ under contaminated and pure items, and let $F_{\mathrm{mix}}=\pi_0F_0+(1-\pi_0)F_1$ be the marginal mixture. The empirical asymmetry that drives \abbr{} is that pure items concentrate near the left tail of $p_i^*$, while the right tail $\{p_i^*>\lambda\}$ is enriched for contaminated items. We exploit this asymmetry by anchoring a conservative envelope to the right tail and propagating it leftward. For a fixed threshold $\lambda\in(0,1)$, define the right-tail conditional CDF
\begin{equation}
\label{eq:right-tail}
    G(x)=\Pp(P^*\le x\mid P^*>\lambda)
    =\frac{F(x)-F(\lambda)}{1-F(\lambda)},\qquad x\in[\lambda,1],
\end{equation}
where $P^*$ denotes a generic draw from the distribution with CDF $F$. Let $g(\lambda^+)$ denote the corresponding right-boundary density. From the observed maxima $\{p_i^*\}_{i=1}^n$, \abbr{} estimates the right-tail mixture CDF $\widehat G_n$ on $[\lambda,1]$ and the boundary density $\widehat g(\lambda^+)$, both implemented by a $k$-nearest-neighbor estimator. The fitted envelope at the anchor is
\begin{equation}
\label{eq:anchor}
    \widehat F_{\mathrm{fit}}(\lambda)
    =
    \frac{\lambda\,\widehat g(\lambda^+)}
    {1+\lambda\,\widehat g(\lambda^+)},
\end{equation}
and the global envelope is the piecewise reconstruction
\begin{equation}
\label{eq:fit-envelope}
    \widehat F_{\mathrm{fit}}(x)
    =
    \begin{cases}
    \dfrac{\widehat F_{\mathrm{fit}}(\lambda)}{\lambda}\,x,
    & x\in(0,\lambda],\\[0.9em]
    \widehat F_{\mathrm{fit}}(\lambda)
    +\bigl(1-\widehat F_{\mathrm{fit}}(\lambda)\bigr)\widehat G_n(x),
    & x\in(\lambda,1].
    \end{cases}
\end{equation}
\begin{wrapfigure}{r}{0.48\linewidth}
\vspace{-0.8em}
\centering
\includegraphics[width=0.95\linewidth]{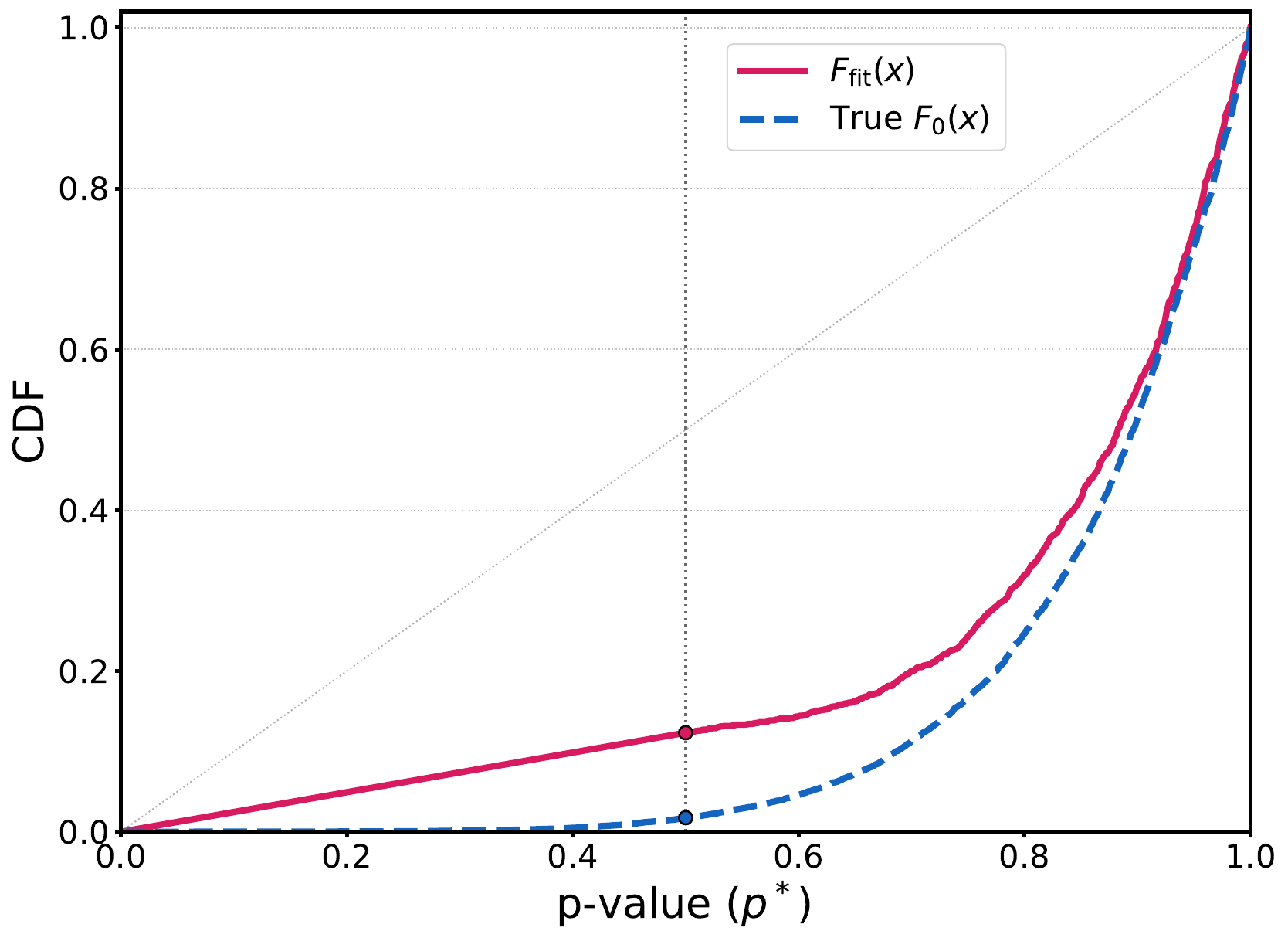}
\caption{Fitted envelope schematic. The fitted $\widehat F_{\mathrm{fit}}$ is a conservative estimate of the true $F_0$.}
\label{fig:synth-cdf-audit}
\vspace{-1.0em}
\end{wrapfigure}
The left branch is linear with slope $\widehat F_{\mathrm{fit}}(\lambda)/\lambda$, whereas the right branch uses the empirical right-tail CDF rescaled to match the anchor. \cref{fig:synth-cdf-audit} illustrates the role of this fitted envelope: it maps the raw max-$p$ values to a conservative CDF scale before BH-style selection. The threshold $\lambda$ marks the point beyond which the right tail is dominated by contaminated items; its data-driven choice and sensitivity are deferred to \cref{sec:experiments}. Each maximum-$p$ value is then transformed through the envelope as $\widetilde p_i=\widehat F_{\mathrm{fit}}(p_i^*)$.

To improve power, we use Storey's estimator \citep{storey2002direct} on the transformed values $\{\widetilde p_i\}$, using the null-proportion estimate
\begin{equation}
\label{eq:pi0}
    \widehat\pi_0
    =
    \frac{\frac{1}{n}\sum_{i=1}^n \mathds{1}\{p_i^*>\lambda\}}
    {1-\widehat F_{\mathrm{fit}}(\lambda)},
\end{equation}
and, with $\widetilde p_{(1)}\le\cdots\le\widetilde p_{(n)}$ the sorted transformed values, returning
\begin{equation}
\label{eq:final-set}
    \cS=\!\left\{i:\widetilde p_i\le\frac{\alpha\,r^{*}}{\widehat\pi_0\,n}\right\},
    \text{ where }
    r^{*}=\max\!\left\{r:\widetilde p_{(r)}\le\frac{\alpha\,r}{\widehat\pi_0\,n}\right\},
\end{equation}
with $r^{*}=0$ and $\cS=\emptyset$ when the set is empty. Relative to \baseline{}, \abbr{} replaces $p_i^*$ with $\widetilde p_i$ to mitigate the conservatism induced by super-uniformity, and replaces the fixed choice $\widehat\pi_0=1$ with the data-adaptive estimate in \cref{eq:pi0}. To establish the theoretical guarantee, we first state the key property motivating the construction of $\widehat F_{\mathrm{fit}}(p_i^*)$:

\begin{proposition}[Envelope domination]
\label{prop:envelope}
Let $F_{\mathrm{fit}}$ be the population counterpart of \Cref{eq:fit-envelope}. Then $F_{\mathrm{fit}}(x)\ge F_0(x)$ for all $x\in(0,1]$.
\end{proposition}

\Cref{prop:envelope} shows that $F_{\mathrm{fit}}$ dominates $F_0$ pointwise. Consequently, the transformed values $\widetilde p_i$ remain super-uniform under the joint null, and the estimator $\widehat\pi_0$ in \cref{eq:pi0} is  conservative for the true null proportion. Combined with the standard Storey-BH argument, these two facts yield asymptotic \metricabbr{} control, as formalized below:

\begin{theorem}[Asymptotic \metricabbr{} control]
\label{thm:fir}
Under the assumptions stated in \Cref{app:assumptions}, the selected set returned by \abbr{} satisfies
\begin{equation}
    \limsup_{n\to\infty}\,\metricop(\cS)\le\alpha.
\end{equation}
\end{theorem}
The proofs of \Cref{prop:envelope,thm:fir} are deferred to \Cref{proof:envelope,proof:fdr}. This guarantee relies on the validity of the model-wise conformal $p$-values under exchangeability between shared-member calibration items and contaminated candidates, the convexity of the null max-$p$ CDF in \Cref{asmp:convex}, and the right-tail distinguishability condition in \Cref{asmp:tail} under which larger max-$p$ values reflect stronger evidence of contamination.

\section{Experimental results}
\label{sec:experiments}

In this section, we empirically evaluate \abbr{} on training-data-detection benchmarks with a controlled member/non-member structure, and verify its joint \metricabbr{} control as well as its power against the \baseline{} baseline of \Cref{sec:scaffold}. Throughout, \abbr{} is instantiated with the data-driven right-tail threshold rule and the Storey null-proportion estimator of \Cref{eq:pi0}, which we adopt as the default configuration.

\subsection{Setup}
\label{sec:setup}

\paragraph{Datasets and audited models.}
The main evaluation uses WikiMIA~\citep{shi2024detecting} and ArXivTection~\citep{duarte2024decop} with $K=16$ independently fine-tuned models initialized from GPT-NeoX-20B~\citep{black-etal-2022-gpt}, Pythia-6.9B~\citep{biderman2023pythia}, and LLaMA-7B~\citep{touvron2023llama}; appendix experiments additionally cover five MIMIR~\citep{duan2024membership} subsets (HackerNews, DMMath, GitHub, Pile-CC, PubMed) in \Cref{app:mimir} and a mixed-family audit pool spanning Pythia ($1.4$B/$2.8$B/$6.9$B/$12$B), GPT-NeoX-20B, and LLaMA ($7$B/$13$B/$30$B) in \Cref{app:mixed-family}. The contamination-controlled three-block split, the resulting per-instance membership labels $M_i^k$, the contamination fraction $\rho$, and per-dataset sample counts are deferred to \Cref{app:data-split}.

\paragraph{Detection scores and baseline.}
We employ the per-model score $T(x;\theta_k)$ with four standard detectors: Perplexity~\citep{carlini2021extracting}, Min-K\%~\citep{shi2024detecting}, Min-K\%++~\citep{zhang2025mink}, and Modified Entropy~\citep{song2021systematic}.

\subsection{Main results}
\label{sec:main-control-power}

\begin{figure}[t]
    \centering
    \includegraphics[width=\textwidth]{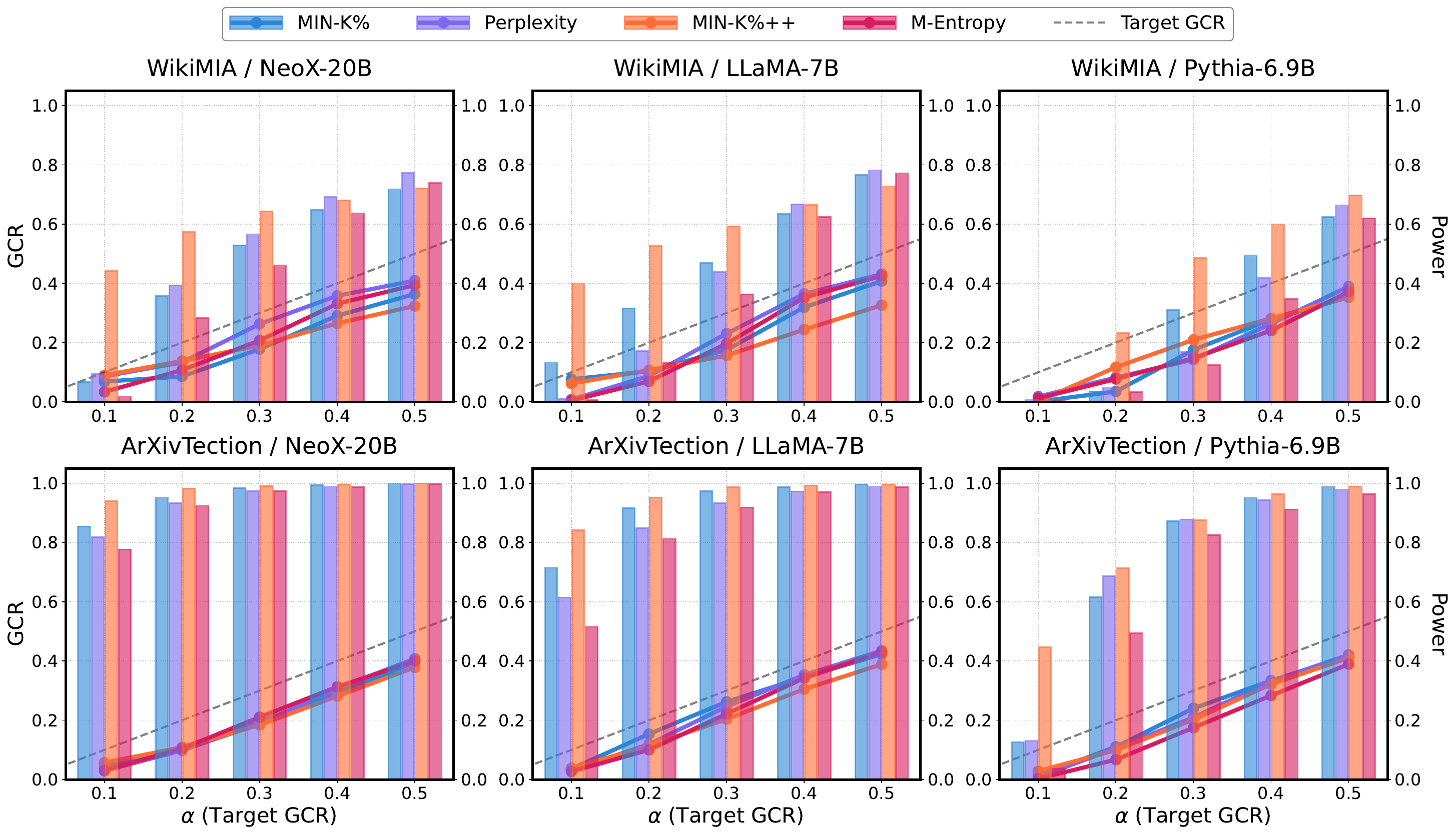}
    \caption{Joint \metricabbr{} control with \abbr{} at $K=16$. Each panel corresponds to a dataset--model family pair and reports realized \fdpabbr{} curves on the left axis and Power bars on the right axis for four detection scores under \abbr{}; the dashed diagonal is the target \metricabbr{}.}
    \label{fig:main-control}
\end{figure}

\paragraph{\abbr{} controls \metricabbr{} across detection scores and model families.}
We next evaluate \abbr{} on the two real benchmarks across all three model families and four detection scores. \Cref{fig:main-control} reports realized \fdpabbr{} curves (left axis) and Power bars (right axis) over the various levels $\alpha$ at $K=16$. Throughout these evaluations, the realized \fdpabbr{} stays at or below the target diagonal. The procedure thus delivers controlled joint pure subsets at the user-specified level without configuration-specific tuning, and the control is stable across detection scores. The same control pattern is verified on five additional MIMIR subsets in \Cref{app:mimir}, and on a heterogeneous mixed-family audit pool of Pythia, GPT-NeoX, and LLaMA variants in \Cref{app:mixed-family}.

\paragraph{\abbr{} attains higher Power than \baseline{}.}
\Cref{tab:power-arxiv} reports the Power comparison between \abbr{} and \baseline{} on ArXivTection at $K=16$. \abbr{} consistently achieves higher Power than \baseline{} across target levels, model families, and detection scores, with the largest gains at smaller $\alpha$. For example, under LLaMA-7B at $\alpha=0.1$ using Min-K\%, \abbr{} improves Power from $0.055$ to $0.715$ over \baseline{}. The corresponding results on WikiMIA  are provided in \Cref{app:power-wikimia}.

\begin{table}[t]
\centering
\caption{Power on ArXivTection at $K=16$. Reported values are average $\mathrm{Power}$ with standard error. ``Base'' refers to \baseline{} and ``Ours'' refers to \abbr{}. \textbf{Bold} numbers mark the better result in each pair.}
\label{tab:power-arxiv}
\vspace{4pt}
\renewcommand\arraystretch{1.2}
\resizebox{\textwidth}{!}{
\begin{tabular}{ll cc cc cc}
\toprule
\multirow{2}{*}{\textbf{Model}} & \multirow{2}{*}{\textbf{Score}} & \multicolumn{2}{c}{$\alpha=0.1$} & \multicolumn{2}{c}{$\alpha=0.2$} & \multicolumn{2}{c}{$\alpha=0.3$} \\
\cmidrule(lr){3-4} \cmidrule(lr){5-6} \cmidrule(lr){7-8}
& & Base & Ours & Base & Ours & Base & Ours \\
\midrule
\multirow{4}{*}{NeoX-20B} & Perplexity & 0.403{\scriptsize$\pm$0.008} & \textbf{0.817{\scriptsize$\pm$0.002}} & 0.829{\scriptsize$\pm$0.001} & \textbf{0.934{\scriptsize$\pm$0.001}} & 0.909{\scriptsize$\pm$0.000} & \textbf{0.973{\scriptsize$\pm$0.000}} \\
& Min-K\%++ & 0.872{\scriptsize$\pm$0.001} & \textbf{0.939{\scriptsize$\pm$0.001}} & 0.944{\scriptsize$\pm$0.000} & \textbf{0.983{\scriptsize$\pm$0.000}} & 0.972{\scriptsize$\pm$0.000} & \textbf{0.991{\scriptsize$\pm$0.000}} \\
& Min-K\% & 0.576{\scriptsize$\pm$0.003} & \textbf{0.854{\scriptsize$\pm$0.002}} & 0.871{\scriptsize$\pm$0.001} & \textbf{0.951{\scriptsize$\pm$0.001}} & 0.926{\scriptsize$\pm$0.000} & \textbf{0.983{\scriptsize$\pm$0.000}} \\
& M-Entropy & 0.235{\scriptsize$\pm$0.008} & \textbf{0.776{\scriptsize$\pm$0.002}} & 0.783{\scriptsize$\pm$0.001} & \textbf{0.924{\scriptsize$\pm$0.001}} & 0.886{\scriptsize$\pm$0.001} & \textbf{0.973{\scriptsize$\pm$0.001}} \\
\midrule
\multirow{4}{*}{LLaMA-7B} & Perplexity & 0.040{\scriptsize$\pm$0.004} & \textbf{0.614{\scriptsize$\pm$0.005}} & 0.579{\scriptsize$\pm$0.002} & \textbf{0.848{\scriptsize$\pm$0.002}} & 0.755{\scriptsize$\pm$0.001} & \textbf{0.934{\scriptsize$\pm$0.002}} \\
& Min-K\%++ & 0.715{\scriptsize$\pm$0.001} & \textbf{0.841{\scriptsize$\pm$0.002}} & 0.844{\scriptsize$\pm$0.001} & \textbf{0.952{\scriptsize$\pm$0.001}} & 0.918{\scriptsize$\pm$0.000} & \textbf{0.986{\scriptsize$\pm$0.000}} \\
& Min-K\% & 0.055{\scriptsize$\pm$0.004} & \textbf{0.715{\scriptsize$\pm$0.003}} & 0.680{\scriptsize$\pm$0.001} & \textbf{0.917{\scriptsize$\pm$0.002}} & 0.793{\scriptsize$\pm$0.001} & \textbf{0.973{\scriptsize$\pm$0.001}} \\
& M-Entropy & 0.026{\scriptsize$\pm$0.003} & \textbf{0.516{\scriptsize$\pm$0.005}} & 0.523{\scriptsize$\pm$0.002} & \textbf{0.813{\scriptsize$\pm$0.002}} & 0.716{\scriptsize$\pm$0.001} & \textbf{0.918{\scriptsize$\pm$0.001}} \\
\midrule
\multirow{4}{*}{Pythia-6.9B} & Perplexity & 0.007{\scriptsize$\pm$0.001} & \textbf{0.130{\scriptsize$\pm$0.006}} & 0.149{\scriptsize$\pm$0.004} & \textbf{0.687{\scriptsize$\pm$0.006}} & 0.464{\scriptsize$\pm$0.004} & \textbf{0.878{\scriptsize$\pm$0.002}} \\
& Min-K\%++ & 0.094{\scriptsize$\pm$0.004} & \textbf{0.447{\scriptsize$\pm$0.002}} & 0.465{\scriptsize$\pm$0.001} & \textbf{0.713{\scriptsize$\pm$0.002}} & 0.627{\scriptsize$\pm$0.001} & \textbf{0.876{\scriptsize$\pm$0.002}} \\
& Min-K\% & 0.006{\scriptsize$\pm$0.001} & \textbf{0.126{\scriptsize$\pm$0.007}} & 0.139{\scriptsize$\pm$0.005} & \textbf{0.616{\scriptsize$\pm$0.005}} & 0.460{\scriptsize$\pm$0.002} & \textbf{0.871{\scriptsize$\pm$0.002}} \\
& M-Entropy & 0.004{\scriptsize$\pm$0.001} & \textbf{0.037{\scriptsize$\pm$0.004}} & 0.095{\scriptsize$\pm$0.004} & \textbf{0.494{\scriptsize$\pm$0.007}} & 0.370{\scriptsize$\pm$0.003} & \textbf{0.826{\scriptsize$\pm$0.002}} \\
\bottomrule
\end{tabular}
}
\end{table}

\subsection{Sensitivity analysis}
\label{sec:sensitivity}

\paragraph{Effect of the right-tail threshold $\lambda$.}
\Cref{fig:lambda-sensitivity} sweeps the right-tail threshold $\lambda\in\{0.3,\ldots,0.9\}$ on WikiMIA and ArXivTection (NeoX-20B, $K=8$, Min-K\%++, $\alpha=0.1$) and overlays the data-driven rule. The adaptive choice tracks the best fixed $\lambda$ in power without producing any \fdpabbr{} violation. This is also the empirical check for the right-tail dominance condition behind \Cref{thm:fir}: even at the smallest $\lambda$ tested, \fdpabbr{} stays below $\alpha$, indicating that the right tail is contaminated-dominated across the operating range.

\begin{figure}[t]
    \centering
    \begin{minipage}[t]{0.49\textwidth}
        \centering
        \includegraphics[width=\textwidth]{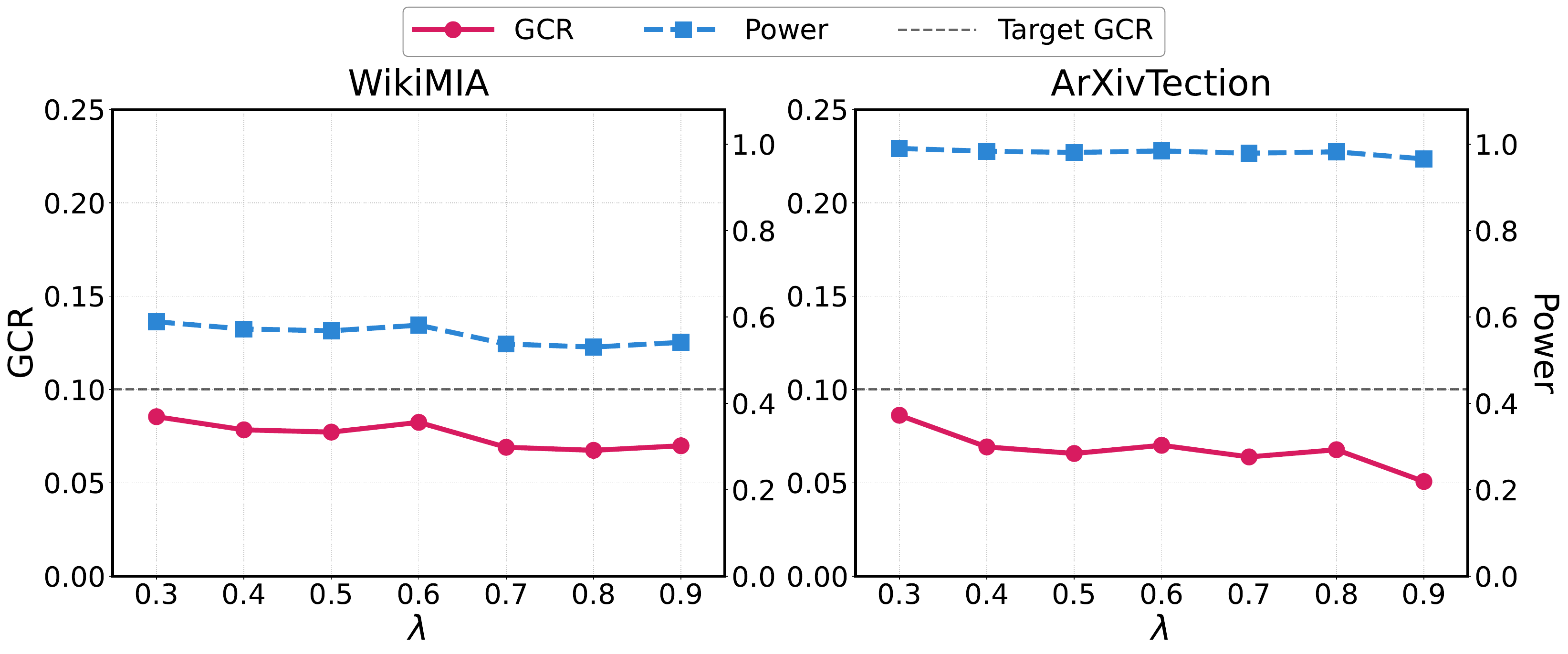}
        \caption{GCR and Power curves over $\lambda$ with $\alpha=0.1$. The dashed line is the target \metricabbr{}.}
        \label{fig:lambda-sensitivity}
    \end{minipage}
    \hfill
    \begin{minipage}[t]{0.49\textwidth}
        \centering
        \includegraphics[width=\textwidth]{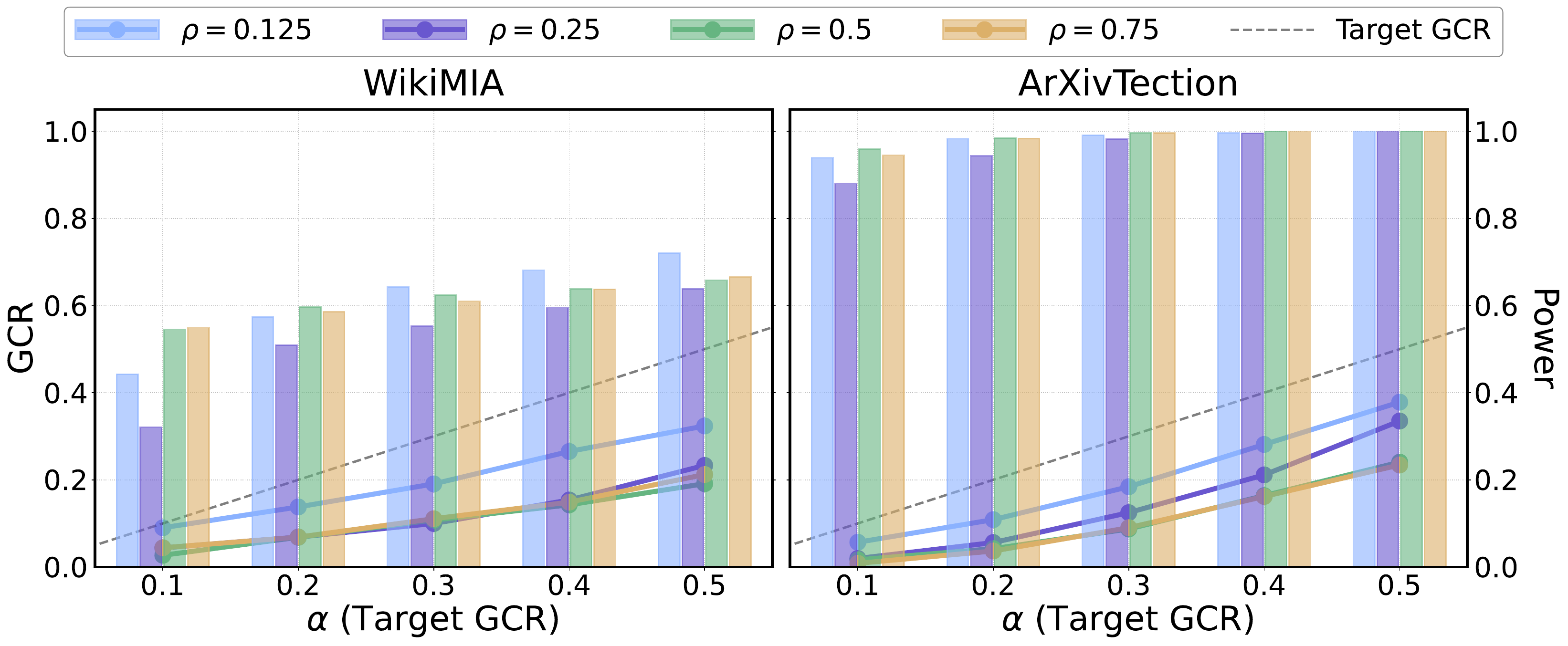}
        \caption{\fdpabbr{} curve (left axis) and Power bars (right axis) under varying training fraction $\rho$.}
        \label{fig:split-sensitivity}
    \end{minipage}
\end{figure}

\paragraph{Impact of the number of audited models $K$ on control and power.}
\Cref{fig:k-sensitivity} varies $K \in \{2, 4, 8, 16\}$ on WikiMIA across all three model families with Min-K\%++. \fdpabbr{} stays under the target diagonal as $K$ grows, while Power does not collapse with $K$. Our claim is not that every Power curve is strictly monotone in $K$, but that the procedure remains controlled as the number of jointly audited models grows; individual power trends can be non-monotone because, by \Cref{eq:joint-purity-prevalence}, the joint-purity prevalence in the test mixture depends on $K$ through the term $(1-\rho)^K$ from the remainder block while the positives in the globally pure block stay fixed.

\begin{figure}[t]
    \centering
    \includegraphics[width=\textwidth]{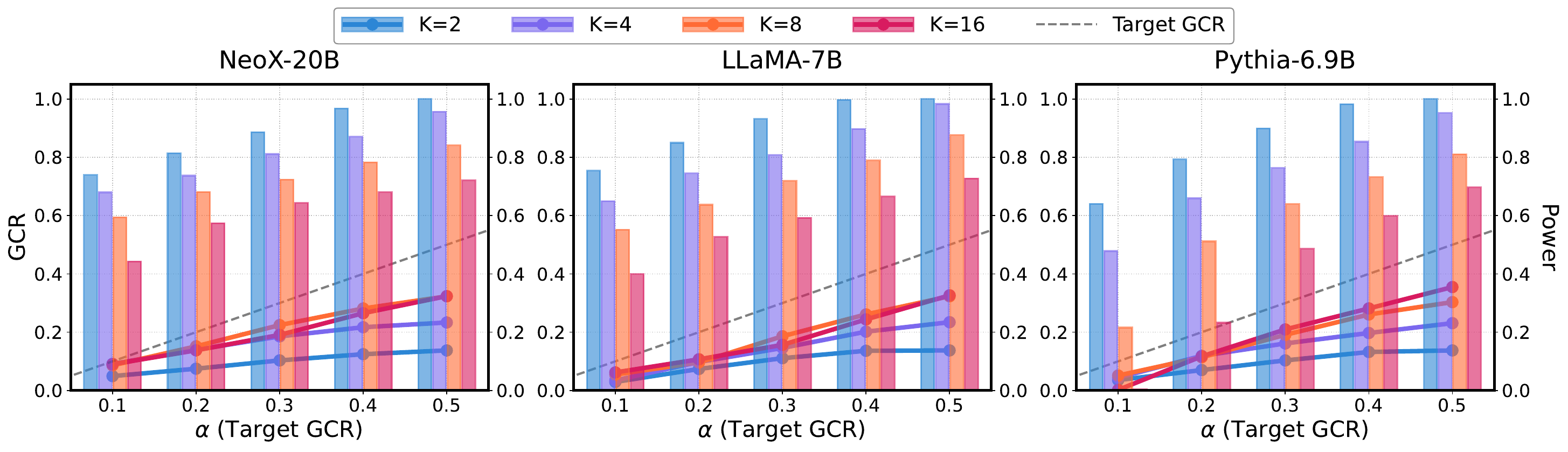}
    \caption{Sensitivity to the number of audited models $K$ on WikiMIA. \fdpabbr{} (lines, left axis) and Power (bars, right axis) with Min-K\%++.}
    \label{fig:k-sensitivity}
\end{figure}

\paragraph{Robustness of error control to variations in the training fraction $\rho$.}
The training fraction $\rho$ controls the contamination level in our experiments: it is the per-model sampling fraction within the remainder block of the controlled three-block split, as defined in \Cref{app:data-split}. \Cref{fig:split-sensitivity} sweeps $\rho \in \{0.125, 0.25, 0.5, 0.75\}$ on NeoX-20B at $K=16$ with Min-K\%++, on both WikiMIA and ArXivTection. Across all four values of $\rho$ and both datasets, the realized \fdpabbr{} stays at or below the target diagonal, confirming that \abbr{}'s control is robust to the contamination level.

\begin{figure}[t]
    \centering
    \begin{minipage}[t]{0.65\textwidth}
        \centering
        \includegraphics[width=\textwidth]{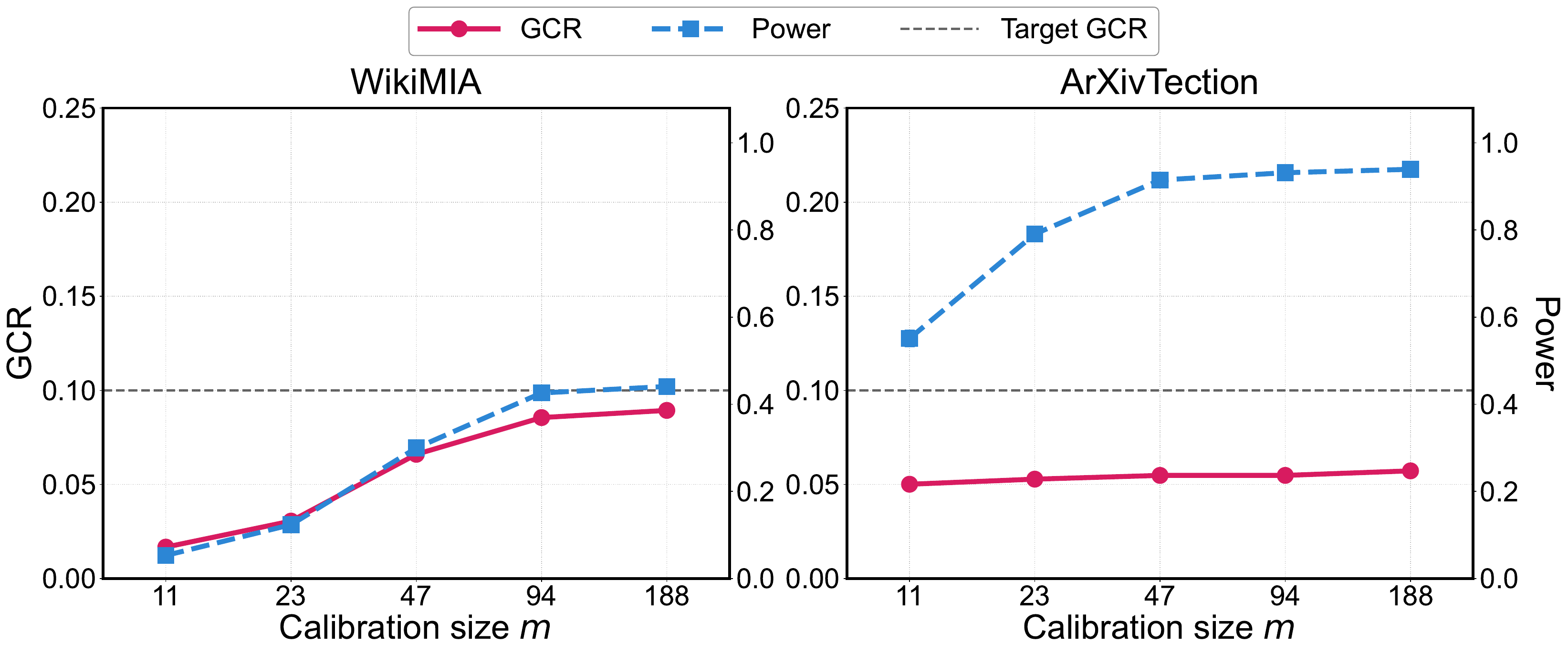}
        \caption{GCR and Power curves over the calibration size $m$ with $\alpha=0.1$. The dashed line is the target \metricabbr{}.}
        \label{fig:calibration-size}
    \end{minipage}
    \hfill
    \begin{minipage}[t]{0.33\textwidth}
        \centering
        \includegraphics[width=\textwidth]{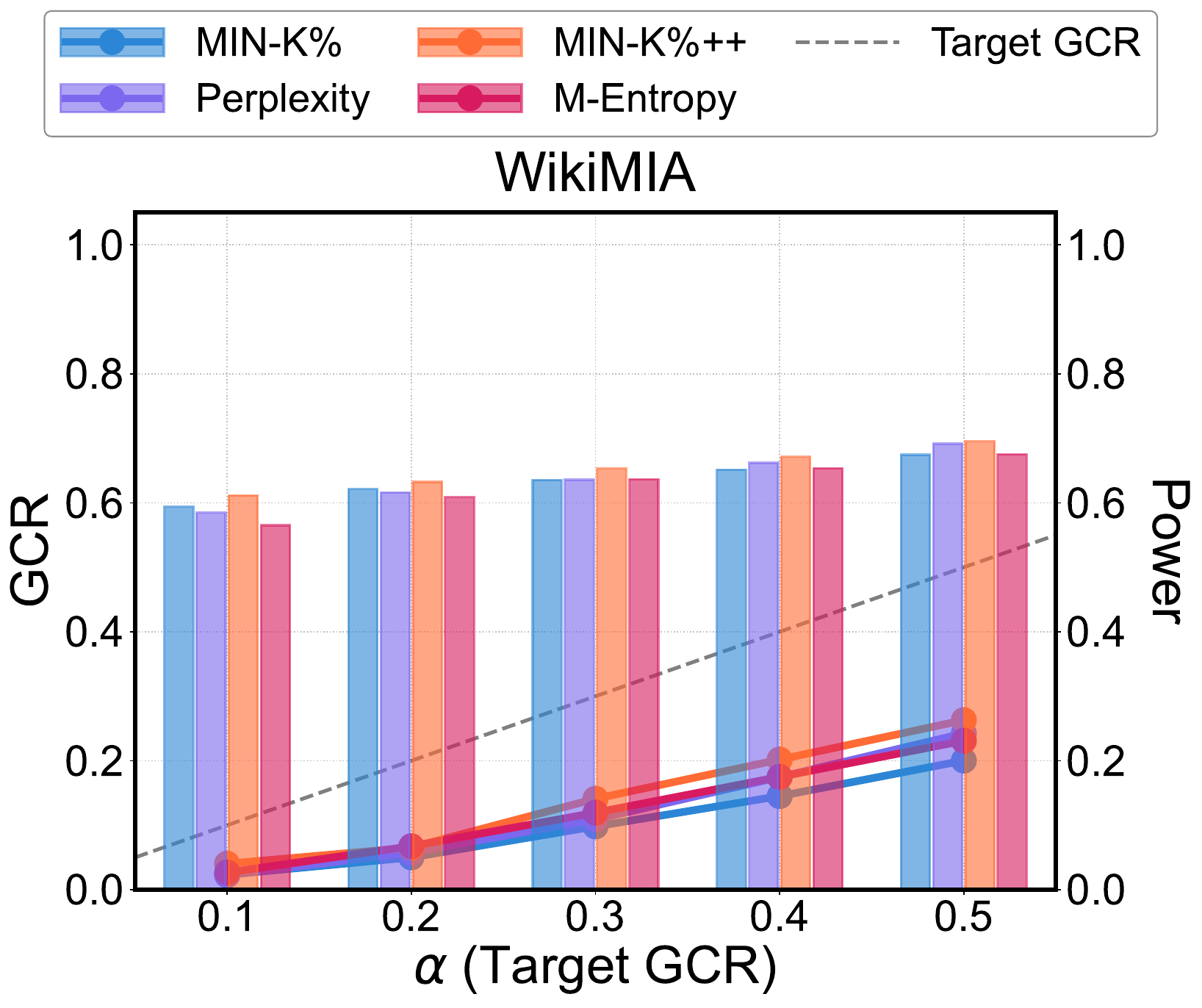}
        \caption{\fdpabbr{} curves (left axis) and Power bars (right axis) on an eight-family audit pool.}
        \label{fig:mixed-family-8}
    \end{minipage}
\end{figure}

\paragraph{Effect of the calibration size $m$.}
\Cref{fig:calibration-size} varies the number of shared calibration items $m\in\{11,23,47,94,188\}$ on WikiMIA and ArXivTection with NeoX-20B at $K=16$, Min-K\%++, and $\alpha=0.1$, where smaller calibration sets are nested subsets of larger ones. The realized \fdpabbr{} stays below $\alpha$ at every calibration size on both datasets. Power increases with $m$ and largely saturates beyond $m=94$: doubling the calibration set to $m=188$ raises power by only $1.5$ and $0.8$ percentage points on WikiMIA and ArXivTection, respectively. A calibration set of around one hundred shared members is therefore sufficient for \abbr{} to reach its full power in these settings.

\paragraph{Robustness to heterogeneous audited model families.}
To test whether \abbr{} requires the audited models to share an architecture, we construct a $K=16$ audit pool on WikiMIA from eight open-weight families, namely Qwen2.5-7B, Mistral-7B, OPT-6.7B, Pythia-6.9B~\citep{biderman2023pythia}, GPT-NeoX-20B~\citep{black-etal-2022-gpt}, LLaMA-7B~\citep{touvron2023llama}, Gemma-2-9B, and OLMo-2-7B, each contributing two adapters fine-tuned on distinct model-specific subsets. As shown in \Cref{fig:mixed-family-8}, the realized \fdpabbr{} stays below the target diagonal for all four detection scores across the evaluated levels, and Power at $\alpha=0.1$ ranges from $0.57$ to $0.61$. The control of \abbr{} therefore carries over to audit pools that mix architectures and pretraining corpora.

\section{Related Work}
\label{sec:related}

\paragraph{Benchmark data contamination.}
Benchmark data contamination occurs when evaluation examples overlap with a model's training data, which can inflate reported performance~\citep{sainz2023nlp,balloccu2024leak,xu2024benchmark,li2024awesome, DBLP:conf/icml/BordtSBL25, cheng2025survey}. One line of work reduces this risk before evaluation through dynamic benchmark construction~\citep{zhu2024dyval, DBLP:conf/aaai/0001GL24,white2025livebench,srivastava2026beyondbench}, data rewriting~\citep{yang2023rethinking,zhu-etal-2024-clean}, or contamination-free benchmark curation~\citep{zhao2024mmlu}. Another line detects contamination after models and benchmarks already exist. Most methods in this setting operate at the sample level, inferring whether a particular example was used for training~\citep{shokri2017membership,carlini2022membership,liu2026how,liu2026provable}; such methods have been studied for LLMs~\citep{carlini2021extracting, shi2024detecting, zhang-etal-2024-pretraining, golchin2024time, li2024taskcontamination, zhang2025finetuning, hu2025a, zhang2025mink, mattern-etal-2023-membership, xie-etal-2024-recall, zhang-etal-2024-pretraining, raoof2025infilling, kaneko-etal-2025-sampling, yi2026membershipinferencellmswild}, vision language models~\citep{ko2023practical, li2024membership, DBLP:conf/uss/HuLLZQ0025, liu2026lomia}.
Dataset-level detection methods instead test for contamination at a broader dataset scale~\citep{DBLP:conf/emnlp/VuHHS23, oren2024proving,ko2023practical,li-etal-2024-open-source, maini2024llm, zhang-etal-2024-pacost,choi2025how, zawalski2026detecting}. Both lines of work target a single audited model or a single dataset and do not return a shared item-level decontaminated subset across multiple audited models. In contrast, \abbr{} aggregates per-model instance-level conformal $p$-values into a joint selection procedure that returns one benchmark shared by $K$ audited models with provable \metricabbr{} control.

\paragraph{Conformal selection.}
Selective inference with false discovery rate control is rooted in the Benjamini--Hochberg (BH) procedure on valid $p$-values~\citep{benjamini1995controlling,benjamini2001control,storey2004strong}. Conformal prediction provides a distribution-free route to such $p$-values under exchangeability~\citep{vovk2005algorithmic,papadopoulos2008inductive}, and recent work combines conformal calibration with BH-type selection to obtain finite-sample selection guarantees~\citep{bates2023testing,jin2023selection}. Subsequent extensions address more constrained settings, including covariate shift through weighted conformal scores~\citep{jin2025model}, multivariate responses through vector-valued conformity scores~\citep{bai2025multivariate}, and multiple response conditions through conditional conformity constructions~\citep{hao2026multicondition}. These ideas have also been applied to LLM alignment~\citep{gui2024conformal}, drug discovery~\citep{bai2025conformal}, and training-data identification~\citep{liu2026provable}.
The closest formal comparison is multivariate conformal selection~\citep{bai2025multivariate}, where the vector structure comes from multiple response coordinates within one prediction problem. In joint benchmark decontamination, the multiplicity comes from audited models: each item has $K$ membership statuses, and it is contaminated if any audited model trained on it. This yields a joint conformal selection problem under a union null: the procedure must aggregate dependent model-wise conformal evidence into one item-level ranking while controlling the expected fraction of union-null items among selected items.

\section{Conclusion}
\label{sec-conclusion}

In this paper, we formalize joint benchmark decontamination as a joint selection problem, where the goal is to select one shared benchmark whose global contamination rate is controlled across all audited models. We propose \method{} (\abbr{}), which computes model-wise conformal $p$-values, aggregates them through a max-$p$ statistic, reconstructs a conservative envelope for the max-$p$ null distribution, and applies adaptive Storey-BH to the transformed values. Experiments on synthetic data and LLM benchmarks show that \abbr{} maintains empirical contamination control across datasets, model families, and detection scores, while improving power over \baseline{}. The method is agnostic to the choice of detection score and applies beyond LLM evaluation whenever shared-member calibration data are available.
\paragraph{Limitations.}
The current procedure relies on a calibration set whose distribution is close to the candidate pool, which is a standard condition for conformal calibration. While importance-weighted conformal $p$-values~\citep{jin2025model} can be incorporated into \abbr{} to handle covariate shift along observed features, establishing formal guarantees under broader distribution shifts remains an important direction for future work. In addition, our formulation focuses on exact instance-level membership, and extending the framework to soft contamination such as near-duplicates, paraphrases, or translated variants is another promising direction.

\begin{ack}
This research is supported by Guangdong Basic and Applied Basic Research Foundation (Grant No. 2026A1515011367)
. This project is also supported by the Jiangsu Provincial Key Discipline Construction Project (Statistics) and the open project of the Joint Lab for Statistics and Finance (Grant No.~2025JLSF101). Hao Zeng is supported by the SUSTech--NUS Joint Research Program. We thank Shenzhen Guangming Science City Development and Construction Co., Ltd.\ and its Material Genome Big-Science Facilities Platform for providing technical support and assistance in data collection and analysis.
\end{ack}

\bibliographystyle{plainnat}
\bibliography{reference}

\newpage

\appendix

\clearpage

\section{Table of Notations}
\label{app:notations}
\begin{table}[H]
\centering
\caption{Summary of notations used in the paper.}
\label{tab:notations}
\renewcommand\arraystretch{1.15}
\setlength{\tabcolsep}{4pt}
\footnotesize
\begin{tabular}{p{0.28\linewidth}p{0.66\linewidth}}
\toprule
\textbf{Notation} & \textbf{Description} \\
\midrule
$\theta_k$ & The $k$-th audited language model, $k\in\{1,\ldots,K\}$. \\
$K$ & The number of audited models. \\
$\cD_{\mathrm{train}}$ & The private training corpus, written $\cD_{\mathrm{train}}(\theta_k)$ for audited model $\theta_k$. \\
$x_i$ & The $i$-th candidate benchmark item; the pool is $\{x_i\}_{i=1}^n$. \\
$\cD_{\mathrm{cal}}$ & The shared calibration set $\{x_\ell\}_{\ell=n+1}^{n+m}$ whose items are known members of every audited model. \\
$n,m$ & The numbers of candidate and calibration items. \\
$M(x;\theta_k),\,M_i^k$ & The binary membership indicator of $x$ (resp.\ candidate $i$) with respect to $\theta_k$, with $M_i^k=M(x_i;\theta_k)$. \\
$H_{0,i},H_{1,i}$ & The joint null and alternative hypotheses: contaminated for at least one audited model, and jointly pure for all audited models. \\
$T(x;\theta_k)$ & The detection score for item $x$ under model $\theta_k$; smaller values indicate stronger evidence of non-membership. \\
$p_i^k$ & The model-wise conformal $p$-value for candidate item $i$ and model $\theta_k$. \\
$p_i^*$ & The aggregated max-$p$ statistic used to test the joint null, $p_i^*=\max_k p_i^k$. \\
$\widetilde p_i$ & The envelope-transformed value $\widehat F_{\mathrm{fit}}(p_i^*)$. \\
$\cS$ & The selected shared benchmark. \\
$\cS_k$ & The model-wise selected set for audited model $\theta_k$. \\
$\fdpop(\cS)$ & The \fdpfull{} (\fdpabbr{}), i.e., the realized fraction of contaminated items in $\cS$. \\
$\metricop(\cS)$ & The \metricfull{}, i.e., $\E[\fdpop(\cS)]$. \\
$\alpha$ & The target \metricabbr{} level. \\
$\mathrm{Power}(\cS)$ & The expected fraction of jointly pure candidate items selected into $\cS$. \\
$F_0,F_1$ & The CDFs of $p_i^*$ conditional on $H_{0,i}$ (contaminated) and $H_{1,i}$ (jointly pure), respectively. \\
$F_{\mathrm{mix}}$ & The marginal CDF of $p_i^*$ over the candidate pool, $F_{\mathrm{mix}}=\pi_0F_0+(1-\pi_0)F_1$. \\
$f_0,f_1,f_{\mathrm{mix}}$ & The densities corresponding to $F_0,F_1,F_{\mathrm{mix}}$. \\
$\lambda$ & The right-tail threshold used to fit the envelope for \abbr{}. \\
$G_0,G_1,G_{\mathrm{mix}}$ & The right-tail conditional CDFs on $[\lambda,1]$ corresponding to $F_0,F_1,F_{\mathrm{mix}}$. \\
$g_0,g_1,g_{\mathrm{mix}}$ & The corresponding right-tail conditional densities. \\
$\widehat G_n$ & The empirical estimate of $G_{\mathrm{mix}}$ computed from $\{p_i^*:p_i^*>\lambda\}$. \\
$\widehat g(\lambda^+)$ & The estimated right-boundary density at $\lambda$. \\
$\widehat F_{\mathrm{fit}}$ & The fitted conservative envelope used to transform max-$p$ values. \\
$\pi_0$ & The population proportion of contaminated candidate items. \\
$\widehat\pi_0$ & The estimated proportion of contaminated candidate items. \\
$\cP,\cP_0,\cP_{\mathrm{cal}},\cP_{\mathrm{rem}}$ & The benchmark candidate pool, clean block, calibration block, and remainder block used in the experimental split. \\
$\rho$ & The per-model training fraction within the remainder block in the controlled experiments. \\
\bottomrule
\end{tabular}
\end{table}

\clearpage

\section{Theoretical analysis}
\label{app:proofs}

This appendix states the formal assumptions used and gives the proof. Let $\pi_0$ and $1-\pi_0$ denote the proportions of contaminated and jointly pure candidate items, respectively. Let $F_0,F_1,F_{\mathrm{mix}}$ and $f_0,f_1,f_{\mathrm{mix}}$ be the corresponding CDFs and densities of $p_i^*$.

\subsection{Assumptions and their justification}
\label{app:assumptions}

\begin{assumption}[Convex null max-$p$ CDF]
\label{asmp:convex}
The null CDF $F_0$ of $p_i^*$ is convex on $[0,1]$.
\end{assumption}

\paragraph{Justification.}
The convexity condition is natural under max aggregation with heterogeneous ensemble exposure. For a contaminated item, let
$\mathcal{K}_i^0=\{k:M_i^k=1\},\qquad c=|\mathcal{K}_i^0|\ge 1$
be the set and number of audited models whose training data contain item $i$. The remaining $K-c$ models treat the item as unseen and tend to produce small conformal $p$-values. Because \abbr{} aggregates by $p_i^*=\max_k p_i^k$, these small values are masked in the right tail. Thus, on the operating interval $[\lambda,1]$, the null CDF is governed by the maximum over the $c$ models whose training data contain item $i$.

Two benchmark dependence regimes illustrate the shape of this CDF. If the null $p$-values from the $c$ models whose training data contain item $i$ are independent and marginally uniform, then
\begin{equation}
    F_{0,c}(x)
    =
    \Pp\!\left(\max_{\ell\le c}p_\ell\le x\right)
    =
    x^c,\qquad x\in[\lambda,1],
\end{equation}
which is strictly convex for $c\ge2$. At the other extreme, if these $c$ null $p$-values are perfectly positively dependent, they move together and the maximum has CDF
\begin{equation}
    F_{0,c}(x)=x,
\end{equation}
which is linear and hence still convex. This case also satisfies the derivative condition used in the proof as an equality.

A positive-dependence model interpolates between these two extremes. Let $C_{\mathrm{co},c}$ denote the comonotonic copula, which represents perfect positive dependence, and let $C_{\mathrm{ind},c}$ denote the independence copula. Under perfect positive dependence, the $c$ marginally uniform null $p$-values can be represented by a shared latent variable $Z\sim\mathrm{Unif}(0,1)$, namely $U_1=\cdots=U_c=Z$. Hence
\begin{equation}
    \Pp(U_1\le u_1,\ldots,U_c\le u_c)
    =
    \Pp\!\left(Z\le \min_{\ell\le c}u_\ell\right)
    =
    \min_{\ell\le c}u_\ell .
\end{equation}
Therefore
\begin{equation}
    C_{\mathrm{co},c}(u_1,\ldots,u_c)=\min_{\ell\le c}u_\ell,
    \qquad
    C_{\mathrm{ind},c}(u_1,\ldots,u_c)=\prod_{\ell=1}^c u_\ell .
\end{equation}
For $\eta\in[0,1]$, consider the convex combination
\begin{equation}
    C_{\eta,c}
    =
    \eta C_{\mathrm{co},c}
    +(1-\eta)C_{\mathrm{ind},c}.
\end{equation}
Evaluating this copula on the diagonal gives
\begin{equation}
    F_{0,c}(x)
    =
    C_{\eta,c}(x,\ldots,x)
    =
    \eta x+(1-\eta)x^c .
\end{equation}
Therefore
\begin{equation}
    F_{0,c}''(x)=(1-\eta)c(c-1)x^{c-2}\ge0
\end{equation}
for $c\ge2$, while the cases $c=1$ or $\eta=1$ reduce to the linear CDF $F_{0,c}(x)=x$. If the exposure frequency varies across contaminated items, the overall null CDF is a mixture
\begin{equation}
    F_0(x)=\sum_{c=1}^K w_c F_{0,c}(x),
    \qquad w_c\ge0,\quad \sum_{c=1}^K w_c=1,
\end{equation}
and convexity is preserved under this mixture. This provides a structural explanation for \cref{asmp:convex}: the max operator removes the influence of models whose training data do not contain the item in the right tail, and the dependence among models whose training data contain the item ranges from independence, which gives $x^c$, to perfect positive dependence, which gives $x$.

The mixture $C_{\eta,c}$ belongs to the positive Fr\'echet family of copulas~\citep{genest2003compound,koike2024invariant}, which models positive dependence as a convex combination of independence and comonotonicity. This family covers common relationships among audited models. For instance, knowledge distillation strengthens teacher--student alignment and moves $F_{0,c}$ toward the comonotonic endpoint $x$ within this family, which preserves convexity. By contrast, sample-dependent interventions applied to only a subset of the audited models, such as targeted unlearning, can induce heterogeneous cross-model dependence outside this family, in which case the null max-$p$ CDF need not be convex and \cref{asmp:convex} may fail.

\begin{assumption}[Distinguishability of pure and contaminated items]
\label{asmp:tail}
On the truncated interval $[\lambda,1]$, the right-tail conditional CDF of pure items is at least that of contaminated items:
\begin{equation}
    G_1(x)\ge G_0(x),\qquad x\in[\lambda,1].
\end{equation}
\end{assumption}

\paragraph{Justification.}
The inequality is a distinguishability statement. Recalling that $G_0,G_1$ are right-tail conditional CDFs (\cref{eq:right-tail}), the assumption $G_1(x)\ge G_0(x)$ on $[\lambda,1]$ is equivalent to
\begin{equation*}
    \Pr\!\big(p_i^*>x\,\big|\,p_i^*>\lambda,\ H_{1,i}\text{ true}\big)
    \;\le\;
    \Pr\!\big(p_i^*>x\,\big|\,p_i^*>\lambda,\ H_{0,i}\text{ true}\big),
    \qquad x\in[\lambda,1],
\end{equation*}
i.e., conditional on having crossed $\lambda$, a pure item is less likely than a contaminated item to escape further into the upper tail at every level $x$. Equivalently, the conditional law of $p_i^*$ given $p_i^*>\lambda$ is stochastically smaller for pure items than for contaminated items, so the two populations remain distinguishable even after both have entered the right tail.

This assumption formalizes a one-sided separation within the right tail. Conditional on having crossed $\lambda$, pure items are expected to produce shallower exceedances, whereas contaminated items are expected to have a heavier upper tail. In the joint setting, this behavior is natural because an item contaminated in even a single audited model $\theta_{k^\star}$ satisfies $p_i^*=\max_k p_i^k\ge p_i^{k^\star}$ and cannot be pulled downward by small $p$-values from the remaining non-member models, whereas a jointly pure item must appear non-member-like across all models. Thus, among items with $p_i^*>\lambda$, pure items concentrate closer to $\lambda$, while contaminated items are more likely to extend toward $1$.

More formally, a standard condition in hypothesis testing for characterizing separation between the null and alternative distributions is the monotone likelihood ratio property, which here requires the density ratio $f_1(x)/f_0(x)$ to be non-increasing so that larger max-$p$ values provide stronger evidence of contamination than of joint purity. The following proposition shows that this condition is sufficient for \Cref{asmp:tail}.

\begin{proposition}[Sufficient condition for \Cref{asmp:tail}]
\label{prop:mlr-tail}
Suppose that $1-F_0(\lambda)>0$ and $1-F_1(\lambda)>0$, and that $f_0(x)>0$ on $[\lambda,1]$ with the likelihood ratio
\begin{equation}
\label{eq:mlr}
    r(x)=\frac{f_1(x)}{f_0(x)}
\end{equation}
non-increasing on $[\lambda,1]$. Then $G_1(x)\ge G_0(x)$ for all $x\in[\lambda,1]$.
\end{proposition}
\begin{proof}
For $j\in\{0,1\}$, the right-tail conditional density on $[\lambda,1]$ is $g_j(t)=f_j(t)/(1-F_j(\lambda))$. Hence the conditional likelihood ratio
\begin{equation}
    \tilde r(t)=\frac{g_1(t)}{g_0(t)}=\frac{1-F_0(\lambda)}{1-F_1(\lambda)}\,r(t)
\end{equation}
is also non-increasing on $[\lambda,1]$. Fix any $x\in[\lambda,1]$. For all $u\in[\lambda,x]$ and $v\in[x,1]$, monotonicity gives $\tilde r(u)\ge\tilde r(x)\ge\tilde r(v)$, and therefore
\begin{equation}
    g_1(u)\,g_0(v)\ge g_0(u)\,\tilde r(x)\,g_0(v)\ge g_0(u)\,g_1(v).
\end{equation}
Integrating this inequality over $u\in[\lambda,x]$ and $v\in[x,1]$ yields
\begin{equation}
    G_1(x)\bigl(1-G_0(x)\bigr)\ge G_0(x)\bigl(1-G_1(x)\bigr),
\end{equation}
which simplifies to $G_1(x)\ge G_0(x)$.
\end{proof}

\begin{proposition}
\label{prop:knn-consistency}
Suppose that the candidate maxima $p_i^*$ are i.i.d. from $F_{\mathrm{mix}}$. The threshold $\lambda\in(0,1)$ is fixed, $1-F_{\mathrm{mix}}(\lambda)>0$, and the conditional right-tail distribution
\begin{equation}
    G_{\mathrm{mix}}(x)=\frac{F_{\mathrm{mix}}(x)-F_{\mathrm{mix}}(\lambda)}
    {1-F_{\mathrm{mix}}(\lambda)},\qquad x\in[\lambda,1],
\end{equation}
has a positive right-boundary density $g_{\mathrm{mix}}(\lambda^+)$ that is continuous in a right neighborhood of $\lambda$. Let
$m_R=\sum_{i=1}^n\ind\{p_i^*>\lambda\}$. The kNN neighborhood size $k_n$ satisfies
\begin{equation}
    k_n\to\infty,\qquad \frac{k_n}{m_R}\xrightarrow{p}0.
\end{equation}
Let $Y_{(1)}\le\cdots\le Y_{(m_R)}$ be the order statistics of the right-tail sample $\{p_i^*:p_i^*>\lambda\}$. Define
\begin{equation}
    \widehat g(\lambda^+)=\frac{k_n}{m_R\,(Y_{(k_n)}-\lambda)},
    \qquad
    \widehat G_n(x)=\frac{1}{m_R}\sum_{i=1}^n
    \ind\{\lambda<p_i^*\le x\},\quad x\in[\lambda,1].
\end{equation}
Then
\begin{equation}
    \widehat g(\lambda^+)\xrightarrow{p}g_{\mathrm{mix}}(\lambda^+),
    \qquad
    \sup_{x\in[\lambda,1]}|\widehat G_n(x)-G_{\mathrm{mix}}(x)|\xrightarrow{p}0.
\end{equation}
\end{proposition}
\begin{proof}
Since $1-F_{\mathrm{mix}}(\lambda)>0$, the law of large numbers gives $m_R/n\xrightarrow{p}1-F_{\mathrm{mix}}(\lambda)$, hence $m_R\xrightarrow{p}\infty$. Conditional on $m_R$, the right-tail sample is i.i.d. from $G_{\mathrm{mix}}$, so the Glivenko--Cantelli theorem gives the uniform convergence of $\widehat G_n$ on $[\lambda,1]$.

It remains to check the boundary density estimator. Let $U_n=G_{\mathrm{mix}}(Y_{(k_n)})$. By the standard uniform order-statistic representation, $U_n/(k_n/m_R)\xrightarrow{p}1$ when $k_n\to\infty$ and $k_n/m_R\to0$. Continuity and positivity of $g_{\mathrm{mix}}(\lambda^+)$ imply
\begin{equation}
    G_{\mathrm{mix}}(\lambda+r)
    =
    g_{\mathrm{mix}}(\lambda^+)r+o(r),
    \qquad r\downarrow0.
\end{equation}
With $r_n=Y_{(k_n)}-\lambda$, the order-statistic result gives $G_{\mathrm{mix}}(\lambda+r_n)=U_n\to0$ in probability, hence $r_n\to0$ in probability and
\begin{equation}
    \widehat g(\lambda^+)
    =
    \frac{k_n/m_R}{r_n}
    =
    \frac{k_n/m_R}{G_{\mathrm{mix}}(\lambda+r_n)}
    \frac{G_{\mathrm{mix}}(\lambda+r_n)}{r_n}
    \xrightarrow{p}g_{\mathrm{mix}}(\lambda^+).
\end{equation}
\end{proof}

\begin{lemma}[Right-tail boundary dominance]
\label{lem:tail-density}
Under \cref{asmp:tail}, the mixture right-tail boundary density satisfies $g_{\mathrm{mix}}(\lambda^+)\ge g_0(\lambda^+)$.
\end{lemma}
\begin{proof}
Write the right-tail mixture density as $g_{\mathrm{mix}}(\lambda^+)=w\,g_0(\lambda^+)+(1-w)g_1(\lambda^+)$, where $w=\pi_0(1-F_0(\lambda))/(1-F_{\mathrm{mix}}(\lambda))\in[0,1]$. \cref{asmp:tail} implies $g_1(\lambda^+)\ge g_0(\lambda^+)$ at the boundary, so the convex combination is at least $g_0(\lambda^+)$.
\end{proof}

\begin{lemma}[Asymptotic \metricabbr{} control of Storey-BH;~\citealp{storey2004strong}, Thm.~4]
\label{lem:storey}
Consider $n$ tests $\{H_{0,i}\}_{i=1}^n$ with input $p$-values $\{p_i\}_{i=1}^n\subset[0,1]$ and order statistics $p_{(1)}\le\cdots\le p_{(n)}$. Recall that $p_i$ is \emph{valid} under $H_{0,i}$ if $\Pp(p_i\le t\mid H_{0,i})\le t$ for all $t\in[0,1]$, as in \cref{eq:model-validity}; assume the inputs are valid in this sense. Let $n_0=|\{i:H_{0,i}\text{ is true}\}|$ denote the number of true nulls, and let $\pi_0\in(0,1]$ denote the limiting null proportion, $n_0/n\to\pi_0$. Let $\widehat F_0$ and $\widehat F$ denote the empirical null and overall CDFs of $\{p_i\}_{i=1}^n$, and assume both converge in probability, uniformly in $t\in[0,1]$, to continuous limits. Let $\widehat\pi_0$ satisfy $\widehat\pi_0\xrightarrow{p}\pi_0^\infty$ for some constant $\pi_0^\infty\ge\pi_0$. For $\alpha\in(0,1)$, the Storey-BH selection
\begin{equation}
\cS=\!\left\{i:p_i\le\frac{\alpha\,r^{*}}{\widehat\pi_0\,n}\right\},
\text{ where }
r^{*}=\max\!\left\{r:p_{(r)}\le\frac{\alpha\,r}{\widehat\pi_0\,n}\right\},
\end{equation}
with $r^{*}=0$ and $\cS=\emptyset$ when the set is empty, satisfies
\begin{equation}
\limsup_{n\to\infty}\E\!\left[\frac{\sum_{i=1}^n\mathds{1}\{i\in\cS,\;H_{0,i}\text{ is true}\}}{1\vee|\cS|}\right]\le\alpha.
\end{equation}
\end{lemma}

\subsection{Proof of \cref{prop:envelope}}
\label{proof:envelope}
\begin{proof}
The conclusion uses \cref{asmp:convex} and \cref{asmp:tail}. We first establish the bound at $x=\lambda$. The map $h(y)=\lambda y/(1+\lambda y)$ is increasing, so \cref{lem:tail-density} gives
\begin{equation}
    F_{\mathrm{fit}}(\lambda)
    \ge
    \frac{\lambda g_0(\lambda^+)}{1+\lambda g_0(\lambda^+)}
    =
    \frac{\lambda f_0(\lambda^+)}{1-F_0(\lambda)+\lambda f_0(\lambda^+)}.
\end{equation}
Let $A=F_0(\lambda)\in[0,1)$ and $B=\lambda f_0(\lambda^+)$. By the convexity of $F_0$ and $F_0(0)=0$, the secant slope $F_0(\lambda)/\lambda$ is at most the right derivative $f_0(\lambda^+)$, so $B\ge A$. Hence $B(1-A)\ge A(1-A)$ and $B/(1-A+B)\ge A$; therefore $F_{\mathrm{fit}}(\lambda)\ge F_0(\lambda)$.

For $x\in(0,\lambda]$, the convexity of $F_0$ on $[0,\lambda]$ with $F_0(0)=0$ implies that $F_0(x)/x$ is nondecreasing, so $F_0(x)\le F_0(\lambda)x/\lambda$. Combined with the bound at $\lambda$,
\begin{equation}
F_{\mathrm{fit}}(x)
=
\frac{F_{\mathrm{fit}}(\lambda)}{\lambda}x
\ge
\frac{F_0(\lambda)}{\lambda}x
\ge
F_0(x).
\end{equation}

For $x\in(\lambda,1]$, write $F_0(x)=F_0(\lambda)+(1-F_0(\lambda))G_0(x)$. \cref{asmp:tail} gives $G_{\mathrm{mix}}(x)\ge G_0(x)$, and we just showed $F_{\mathrm{fit}}(\lambda)\ge F_0(\lambda)$. The map $(a,b)\mapsto a+(1-a)b$ is nondecreasing in $a\in[0,1)$ for fixed $b\in[0,1]$ and nondecreasing in $b$ for fixed $a$, so
\begin{equation}
F_{\mathrm{fit}}(x)
=
F_{\mathrm{fit}}(\lambda)+\bigl(1-F_{\mathrm{fit}}(\lambda)\bigr)G_{\mathrm{mix}}(x)
\ge
F_0(\lambda)+\bigl(1-F_0(\lambda)\bigr)G_0(x)
=
F_0(x).\qedhere
\end{equation}
\end{proof}

\subsection{Proof of \cref{thm:fir}}
\label{proof:fdr}
By \cref{prop:knn-consistency}, the fitted envelope converges to $F_{\mathrm{fit}}$ uniformly on the two pieces of \cref{eq:fit-envelope}. By \cref{prop:envelope}, $F_{\mathrm{fit}}$ dominates $F_0$. For a null item and any $t\in(0,1)$,
\begin{equation}
\label{eq:null-superuniform}
    \Pp(F_{\mathrm{fit}}(p_i^*)\le t\mid H_{0,i})
    =
    \Pp(p_i^*\le F_{\mathrm{fit}}^{-1}(t)\mid H_{0,i})
    \le
    \Pp(p_i^*\le F_0^{-1}(t)\mid H_{0,i})
    \le t,
\end{equation}
so the transformed null values are asymptotically super-uniform. The estimator \cref{eq:pi0} is asymptotically conservative because
\begin{equation}
\label{eq:pi0-limit}
    \widehat\pi_0
    \xrightarrow{p}
    \frac{1-F_{\mathrm{mix}}(\lambda)}{1-F_{\mathrm{fit}}(\lambda)}
    \ge
    \frac{\pi_0(1-F_0(\lambda))}{1-F_0(\lambda)}
    =
    \pi_0,
\end{equation}
using $F_{\mathrm{fit}}(\lambda)\ge F_0(\lambda)$.

We now invoke \cref{lem:storey} on the transformed inputs $\widetilde p_i=\widehat F_{\mathrm{fit}}(p_i^*)$ with the selection in \cref{eq:final-set}. \cref{eq:null-superuniform} verifies the validity of $\{\widetilde p_i\}$ under the null (in the asymptotic sense, since $\widehat F_{\mathrm{fit}}\to F_{\mathrm{fit}}$ uniformly by \cref{prop:knn-consistency}). The Glivenko--Cantelli theorem applied to the i.i.d.\ null sub-sample and to the full i.i.d.\ sample $\{p_i^*\}\sim F_{\mathrm{mix}}$, combined with the monotonicity of $\widehat F_{\mathrm{fit}}$, gives the required uniform convergence of the null and overall empirical CDFs. The conservative null-proportion limit is supplied by \cref{eq:pi0-limit}. \cref{lem:storey} therefore yields $\limsup_{n\to\infty}\metricop(\cS)\le\alpha$. \qed

\section{Algorithm Summary for \baseline{}}
\label{app:jmcs}

\Cref{alg:jmcs} summarizes \baseline{}, the max-$p$ baseline used in our experiments. It computes model-wise conformal $p$-values, aggregates them by $p_i^*=\max_k p_i^k$, and applies BH to the aggregated values to return the shared selected benchmark.

\begin{algorithm}[h]
\caption{\baselinefull{} (\baseline{})}
\label{alg:jmcs}
\begin{algorithmic}[1]
\REQUIRE Candidate items $\{x_i\}_{i=1}^n$, audited models $\{\theta_k\}_{k=1}^K$, calibration data $\cD_{\mathrm{cal}}=\{x_i\}_{i=n+1}^{n+m}$, target level $\alpha$.
\FOR{$k=1,\ldots,K$}
    \STATE Construct conformal $p$-values $\{p_i^k\}_{i=1}^n$ via \cref{eq:prelim-model-p}.
\ENDFOR
\STATE Aggregate $p_i^*=\max_{1\le k\le K}p_i^k$ for $i=1,\ldots,n$.
\STATE Sort the maxima $p_{(1)}^*\le\cdots\le p_{(n)}^*$ and compute $r^{*}=\max\{r:p_{(r)}^*\le\alpha\,r/n\}$.
\STATE Return $\cS=\{i:p_i^*\le\alpha\,r^{*}/n\}$.
\end{algorithmic}
\end{algorithm}

\section{Experimental details}
\label{app:exp-details}

Our experiments obtain ground-truth membership labels by controlled fine-tuning. For each audited model, an item is labeled as a member if it is included in that model's fine-tuning corpus, and as a non-member otherwise.

\subsection{Tail-threshold selection}
\label{app:lambda-selection}

In all default runs of \abbr{}, we choose the right-tail threshold $\lambda$ from the grid $\{0.5,0.6,0.7,0.8,0.9\}$. The criterion is to minimize the fitted left-branch slope. For each candidate $\lambda$, we estimate $\widehat g(\lambda^+)$ from the right-tail sample $\{p_i^*:p_i^*>\lambda\}$ and form the anchor
\begin{equation}
    \widehat F_{\mathrm{fit}}(\lambda)
    =
    \frac{\lambda\,\widehat g(\lambda^+)}
    {1+\lambda\,\widehat g(\lambda^+)}.
\end{equation}
We then set
\begin{equation}
\label{eq:lambda-selection}
    \widehat\lambda
    =
    \arg\min_{\lambda}
    \frac{\widehat F_{\mathrm{fit}}(\lambda)}{\lambda}
    =
    \arg\min_{\lambda}
    \frac{\widehat g(\lambda^+)}
    {1+\lambda\,\widehat g(\lambda^+)}.
\end{equation}
This rule selects the threshold whose fitted left branch has the smallest slope. The selected $\widehat\lambda$ is then used in \cref{eq:fit-envelope} to construct $\widehat F_{\mathrm{fit}}$, transform the max-$p$ values, and estimate $\widehat\pi_0$.

\subsection{Dataset split setting}
\label{app:data-split}

This subsection explains how we construct the controlled member/non-member structure that supplies the ground-truth joint-contamination labels used in our experiments.

Starting from the candidate pool $\cP$ of a benchmark, we split the items into three disjoint parts: a \emph{clean block} $\cP_0$, a \emph{calibration block} $\cP_{\mathrm{cal}}$, and a \emph{remainder block} $\cP_{\mathrm{rem}}$, with sizes proportional to $a$, $b$, and $1 - a - b$ respectively. For each item $x_i$ and each audited model $\theta_k$, the membership label $M_i^k \in \{0, 1\}$ is then drawn block-by-block:
\begin{equation}
\label{eq:joint-purity-prevalence}
\Pp(M_i^k = 0) \;=\;
\begin{cases}
1, & x_i \in \cP_0,\\
0, & x_i \in \cP_{\mathrm{cal}},\\
1 - \rho, & x_i \in \cP_{\mathrm{rem}}.
\end{cases}
\end{equation}
The three blocks play distinct roles. Items in $\cP_0$ are never used to fine-tune any audited model, so every item in this block is jointly pure by construction. Items in $\cP_{\mathrm{cal}}$ are included in the fine-tuning corpus of \emph{every} audited model and are reused as the shared calibration set $\cD_{\mathrm{cal}}$ in \Cref{sec:scaffold}. Items in $\cP_{\mathrm{rem}}$ supply the per-model contamination. For each audited model $\theta_k$, we independently sample a fraction $\rho$ of the remainder uniformly at random and add those items to $\theta_k$'s fine-tuning corpus. We call $\rho$ the \emph{training fraction}; it is the per-model sampling fraction \emph{within the remainder block}, not within the full pool. Because the assignments on $\cP_{\mathrm{rem}}$ are independent across the $K$ audited models, the joint-purity probability $\Pp(\bigvee_{k=1}^{K} M_i^k = 0)$ equals $1$ on $\cP_0$ and $(1 - \rho)^K$ on $\cP_{\mathrm{rem}}$. Selection is performed on $\cD_{\mathrm{test}} := \cP_0 \cup \cP_{\mathrm{rem}}$, with the calibration block held out. Throughout the main evaluation we use $a = b = 0.3$ and $\rho = 0.125$, and we vary $\rho$ in the sensitivity analysis of \Cref{sec:sensitivity}.

\begin{table}[h]
\centering
\caption{Dataset split statistics at $\rho = 0.125$. Counts are reported for a single audited model $\theta_k$: ``Member'' denotes $M_i^k = 1$ and ``Non-member'' denotes $M_i^k = 0$. Calibration items are shared members of all audited models, and the per-model test-set counts are identical across $k$.}
\label{tab:dataset-stats}
\vspace{4pt}
\begin{tabular}{llccc}
\toprule
\textbf{Dataset} & \textbf{Component} & \textbf{Member} & \textbf{Non-member} & \textbf{Total} \\
\midrule
\multirow{3}{*}{WikiMIA} & Test set $\cD_{\mathrm{test}}$ & 39 & 514 & 553 \\
 & Calibration set $\cD_{\mathrm{cal}}$ & 236 & 0 & 236 \\
 & Candidate pool $\cP$ & 275 & 514 & 789 \\
\midrule
\multirow{3}{*}{ArXivTection} & Test set $\cD_{\mathrm{test}}$ & 39 & 512 & 551 \\
 & Calibration set $\cD_{\mathrm{cal}}$ & 235 & 0 & 235 \\
 & Candidate pool $\cP$ & 274 & 512 & 786 \\
\midrule
\multirow{3}{*}{HackerNews} & Test set $\cD_{\mathrm{test}}$ & 50 & 650 & 700 \\
 & Calibration set $\cD_{\mathrm{cal}}$ & 300 & 0 & 300 \\
 & Candidate pool $\cP$ & 350 & 650 & 1000 \\
\midrule
\multirow{3}{*}{DM-Math} & Test set $\cD_{\mathrm{test}}$ & 50 & 650 & 700 \\
 & Calibration set $\cD_{\mathrm{cal}}$ & 300 & 0 & 300 \\
 & Candidate pool $\cP$ & 350 & 650 & 1000 \\
\midrule
\multirow{3}{*}{GitHub} & Test set $\cD_{\mathrm{test}}$ & 50 & 650 & 700 \\
 & Calibration set $\cD_{\mathrm{cal}}$ & 300 & 0 & 300 \\
 & Candidate pool $\cP$ & 350 & 650 & 1000 \\
\midrule
\multirow{3}{*}{Pile-CC} & Test set $\cD_{\mathrm{test}}$ & 50 & 650 & 700 \\
 & Calibration set $\cD_{\mathrm{cal}}$ & 300 & 0 & 300 \\
 & Candidate pool $\cP$ & 350 & 650 & 1000 \\
\midrule
\multirow{3}{*}{PubMed Central} & Test set $\cD_{\mathrm{test}}$ & 50 & 650 & 700 \\
 & Calibration set $\cD_{\mathrm{cal}}$ & 300 & 0 & 300 \\
 & Candidate pool $\cP$ & 350 & 650 & 1000 \\
\bottomrule
\end{tabular}
\end{table}

\paragraph{Evaluation protocol.}
All reported \fdpabbr{} and Power values are Monte-Carlo means over 500 repetitions on the fixed three-block partition of \Cref{app:data-split}. In each repetition we draw $80\%$ of $\cD_{\mathrm{cal}}$ and $80\%$ of $\cD_{\mathrm{test}}$ uniformly at random; the three-block split itself is held fixed across repetitions. 

\subsection{Fine-tuning setting}
\label{app:finetune}

We fine-tune all audited models using LoRA-SFT. The LoRA adapter is configured with rank $r=16$, scaling parameter $\alpha_{\mathrm{LoRA}}=32$, dropout $0.05$, no bias terms, and task type \texttt{CAUSAL\_LM}. We train all audited models with AdamW using learning rate $10^{-3}$, weight decay $5\times 10^{-4}$, $10$ warmup steps, a cosine learning-rate schedule, and per-device batch size $8$ for both training and evaluation. The number of fine-tuning epochs is backbone-specific: GPT-NeoX-20B is trained for $5$ epochs, while the Pythia and LLaMA backbones are trained for $3$ epochs.

\subsection{Computational cost}
\label{app:compute-budget}

All audited models are fine-tuned with the LoRA-SFT protocol of \cref{app:finetune}, using PyTorch FSDP on NVIDIA RTX 4090 GPUs. Fine-tuning one audited model on one dataset takes about $5$ GPU-min for LLaMA-7B or Pythia-6.9B ($3$ epochs, $2$ GPUs), and about $15$ GPU-min for GPT-NeoX-20B ($5$ epochs, $4$ GPUs). After the scores $T(x;\theta_k)$ are extracted, \abbr{} only performs NumPy operations on the $(n_{\mathrm{test}}, K)$ score matrix, including conformal $p$-value computation, max-$p$ aggregation, envelope fitting, Storey estimation, and Storey-BH; every evaluated setting in \cref{sec:experiments} runs in under one minute on a single CPU core and requires no GPU.

\section{Additional Experimental Results}
\label{app:extra-experiments}

\subsection{Synthetic setup for naive composition failure}
\label{app:synthetic-motivation}

We generate a pool of $1200$ candidate items for $K=4$ audited models $\{\theta_k\}_{k=1}^K$, and split it into a calibration set $\cD_{\mathrm{cal}}$ of size $n_{\mathrm{cal}}=360$ ($30\%$ of the pool) and an audit set $\cD_{\mathrm{test}}$ of size $n_{\mathrm{test}}=840$. Items in $\cD_{\mathrm{cal}}$ are members of every audited model by construction, matching the shared-member calibration assumption of \Cref{sec:scaffold}. Items in $\cD_{\mathrm{test}}$ receive independent Bernoulli labels $M_i^k\stackrel{\mathrm{iid}}{\sim}\mathrm{Bernoulli}(0.30)$ across both $i$ and $k$, so the joint-purity probability of an audit item is $(1-0.30)^K\approx 0.24$. For each item $i$ and audited model $\theta_k$, the detection score is
\begin{equation}
\label{eq:synthetic-score}
T_i^k = \mu\cdot M_i^k + \varepsilon_i^k,\qquad \varepsilon_i^k\stackrel{\mathrm{iid}}{\sim}\cN(0,1),
\end{equation}
with signal strength $\mu=4$, so members and non-members are $\cN(\mu,1)$ and $\cN(0,1)$, and the model-wise scores are independent across $k$.

\subsection{MIMIR Subsets}
\label{app:mimir}

We further evaluate \abbr{} on five subsets (HackerNews, DMMath, GitHub, Pile-CC, PubMed) of the MIMIR benchmark~\citep{duan2024membership}, fine-tuning $K=8$ audited models from Pythia-6.9B and LLaMA-7B per subset under the same protocol as \cref{sec:setup}. \cref{fig:mimir-robust} reports realized \fdpabbr{} and Power for the four detection scores in \cref{sec:setup} at $\alpha\in\{0.1,0.2,0.3,0.4,0.5\}$. The realized \fdpabbr{} stays at or below the target diagonal throughout, confirming that \metricabbr{} control extends beyond the WikiMIA and ArXivTection candidate pools used in \cref{sec:experiments}.

\begin{figure}[h]
    \centering
    \begin{subfigure}{\textwidth}
        \centering
        \includegraphics[width=\textwidth]{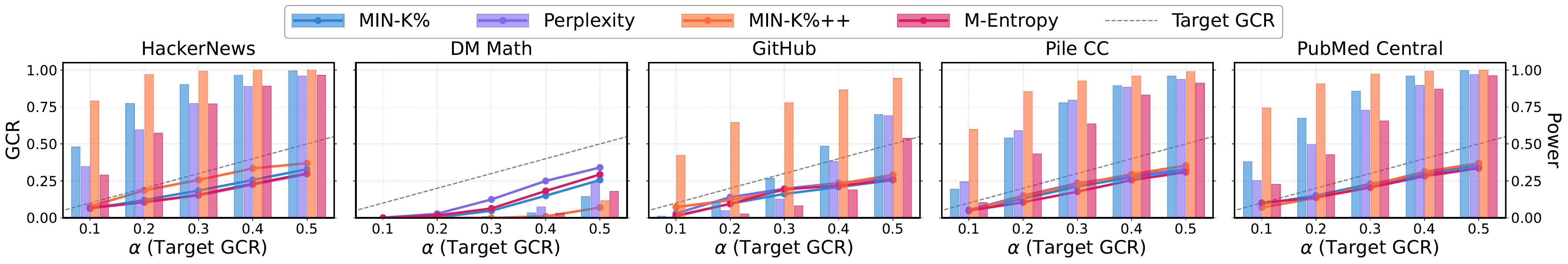}
        \caption{Pythia-6.9B}
    \end{subfigure}
    \\[4pt]
    \begin{subfigure}{\textwidth}
        \centering
        \includegraphics[width=\textwidth]{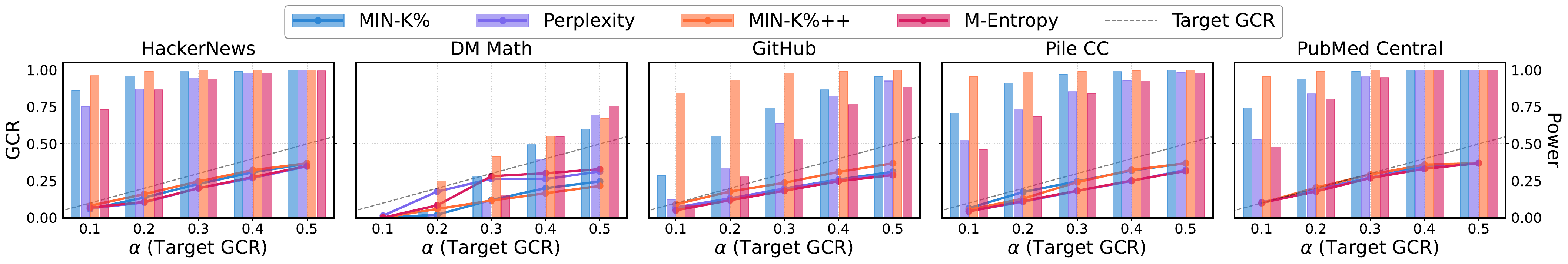}
        \caption{LLaMA-7B}
    \end{subfigure}
    \caption{Realized \fdpabbr{} and Power of \abbr{} on five MIMIR subsets at $K=8$ across Pythia-6.9B and LLaMA-7B. \fdpabbr{} curves (lines, left axis) and Power bars (right axis) for four detection scores at $\alpha\in\{0.1,0.2,0.3,0.4,0.5\}$ on HackerNews, DMMath, GitHub, Pile-CC, and PubMed; the dashed line is the target \metricabbr{}.}
    \label{fig:mimir-robust}
\end{figure}

\subsection{Power on WikiMIA}
\label{app:power-wikimia}

\Cref{tab:power-wikimia} reports Power for \abbr{} and \baseline{} on WikiMIA at $K=16$, mirroring the ArXivTection table of \Cref{tab:power-arxiv}. \abbr{} attains higher Power than \baseline{} across the evaluated targets, model families, and detection scores, with the largest gain on LLaMA-7B with Perplexity at $\alpha=0.3$, where Power rises from $0.088$ to $0.439$. This pattern matches the aggregate finding of \Cref{sec:experiments} and the super-uniformity tax interpretation of \Cref{sec:scaffold}.

\begin{table}[t]
\centering
\caption{Power on WikiMIA at $K=16$. Reported values are average $\mathrm{Power}$ with standard error. ``Base'' refers to \baseline{} and ``Ours'' refers to \abbr{}. \textbf{Bold} numbers mark the better result in each pair.}
\label{tab:power-wikimia}
\vspace{4pt}
\renewcommand\arraystretch{1.2}
\resizebox{\textwidth}{!}{
\begin{tabular}{ll cc cc cc}
\toprule
\multirow{2}{*}{\textbf{Model}} & \multirow{2}{*}{\textbf{Score}} & \multicolumn{2}{c}{$\alpha=0.1$} & \multicolumn{2}{c}{$\alpha=0.2$} & \multicolumn{2}{c}{$\alpha=0.3$} \\
\cmidrule(lr){3-4} \cmidrule(lr){5-6} \cmidrule(lr){7-8}
& & Base & Ours & Base & Ours & Base & Ours \\
\midrule
\multirow{4}{*}{NeoX-20B} & Perplexity & 0.000{\scriptsize$\pm$0.000} & \textbf{0.094{\scriptsize$\pm$0.005}} & 0.089{\scriptsize$\pm$0.003} & \textbf{0.394{\scriptsize$\pm$0.004}} & 0.326{\scriptsize$\pm$0.002} & \textbf{0.565{\scriptsize$\pm$0.005}} \\
& Min-K\%++ & 0.382{\scriptsize$\pm$0.001} & \textbf{0.443{\scriptsize$\pm$0.001}} & 0.499{\scriptsize$\pm$0.001} & \textbf{0.574{\scriptsize$\pm$0.002}} & 0.580{\scriptsize$\pm$0.001} & \textbf{0.643{\scriptsize$\pm$0.001}} \\
& Min-K\% & 0.006{\scriptsize$\pm$0.001} & \textbf{0.067{\scriptsize$\pm$0.003}} & 0.152{\scriptsize$\pm$0.003} & \textbf{0.358{\scriptsize$\pm$0.003}} & 0.358{\scriptsize$\pm$0.001} & \textbf{0.529{\scriptsize$\pm$0.003}} \\
& M-Entropy & 0.000{\scriptsize$\pm$0.000} & \textbf{0.017{\scriptsize$\pm$0.002}} & 0.041{\scriptsize$\pm$0.002} & \textbf{0.283{\scriptsize$\pm$0.004}} & 0.229{\scriptsize$\pm$0.003} & \textbf{0.460{\scriptsize$\pm$0.005}} \\
\midrule
\multirow{4}{*}{LLaMA-7B} & Perplexity & 0.000{\scriptsize$\pm$0.000} & \textbf{0.010{\scriptsize$\pm$0.001}} & 0.010{\scriptsize$\pm$0.001} & \textbf{0.171{\scriptsize$\pm$0.005}} & 0.088{\scriptsize$\pm$0.002} & \textbf{0.439{\scriptsize$\pm$0.007}} \\
& Min-K\%++ & 0.120{\scriptsize$\pm$0.005} & \textbf{0.400{\scriptsize$\pm$0.002}} & 0.456{\scriptsize$\pm$0.001} & \textbf{0.527{\scriptsize$\pm$0.001}} & 0.526{\scriptsize$\pm$0.001} & \textbf{0.592{\scriptsize$\pm$0.002}} \\
& Min-K\% & 0.018{\scriptsize$\pm$0.001} & \textbf{0.133{\scriptsize$\pm$0.002}} & 0.178{\scriptsize$\pm$0.001} & \textbf{0.316{\scriptsize$\pm$0.002}} & 0.301{\scriptsize$\pm$0.001} & \textbf{0.469{\scriptsize$\pm$0.003}} \\
& M-Entropy & 0.000{\scriptsize$\pm$0.000} & \textbf{0.006{\scriptsize$\pm$0.001}} & 0.007{\scriptsize$\pm$0.001} & \textbf{0.130{\scriptsize$\pm$0.004}} & 0.064{\scriptsize$\pm$0.002} & \textbf{0.363{\scriptsize$\pm$0.006}} \\
\midrule
\multirow{4}{*}{Pythia-6.9B} & Perplexity & 0.000{\scriptsize$\pm$0.000} & \textbf{0.008{\scriptsize$\pm$0.001}} & 0.012{\scriptsize$\pm$0.001} & \textbf{0.048{\scriptsize$\pm$0.002}} & 0.043{\scriptsize$\pm$0.001} & \textbf{0.168{\scriptsize$\pm$0.006}} \\
& Min-K\%++ & 0.000{\scriptsize$\pm$0.000} & \textbf{0.003{\scriptsize$\pm$0.001}} & 0.010{\scriptsize$\pm$0.001} & \textbf{0.232{\scriptsize$\pm$0.007}} & 0.181{\scriptsize$\pm$0.004} & \textbf{0.486{\scriptsize$\pm$0.003}} \\
& Min-K\% & 0.000{\scriptsize$\pm$0.000} & \textbf{0.001{\scriptsize$\pm$0.000}} & 0.003{\scriptsize$\pm$0.001} & \textbf{0.036{\scriptsize$\pm$0.003}} & 0.029{\scriptsize$\pm$0.002} & \textbf{0.311{\scriptsize$\pm$0.007}} \\
& M-Entropy & 0.000{\scriptsize$\pm$0.000} & \textbf{0.004{\scriptsize$\pm$0.001}} & 0.011{\scriptsize$\pm$0.001} & \textbf{0.034{\scriptsize$\pm$0.002}} & 0.031{\scriptsize$\pm$0.001} & \textbf{0.125{\scriptsize$\pm$0.005}} \\
\bottomrule
\end{tabular}
}
\end{table}

\subsection{Cross-experiment variability of \abbr{} on Min-K\%++}
\label{app:minkpp-stddev}

\Cref{fig:minkpp-fdr-power-std} shows the cross-experiment variability of \abbr{} with Min-K\%++ at $K=16$. The \fdpabbr{} remains aligned with the target diagonal, and the shaded band is narrow, with its upper edge only slightly above the target at smaller $\alpha$. Power also varies little, supporting the stability of the gains reported in \cref{tab:power-arxiv}.

\begin{figure}[h]
    \centering
    \includegraphics[width=\textwidth]{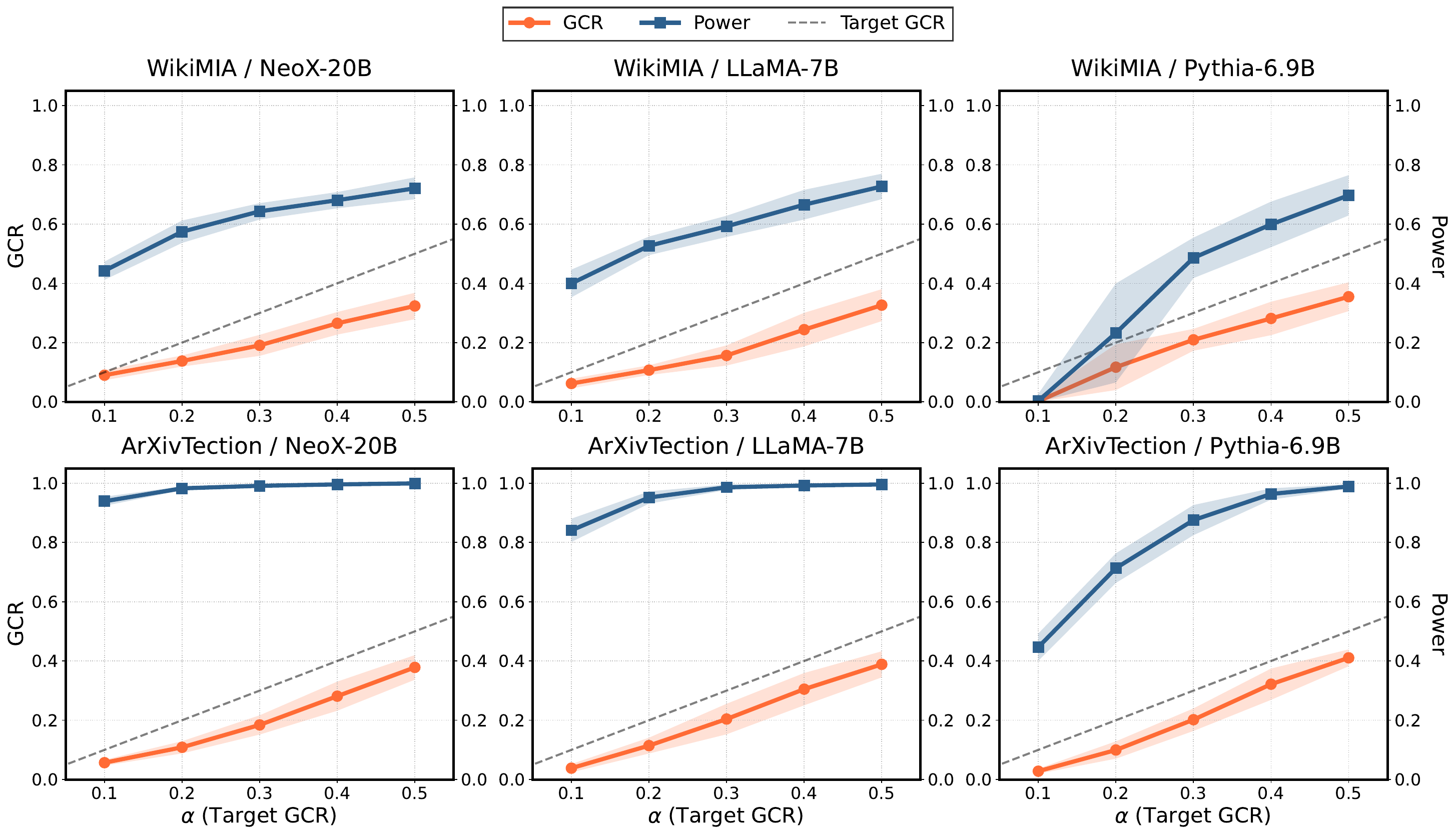}
    \caption{\abbr{} with Min-K\%++ at $K=16$: cross-experiment dispersion. Curves report mean realized \fdpabbr{} (left axis, orange) and Power (right axis, blue) on WikiMIA (top row) and ArXivTection (bottom row), for GPT-NeoX-$20$B, LLaMA-$7$B, and Pythia-$6.9$B. Shaded bands are $\pm 1$ standard deviation, clipped to $[0,1]$; the dashed diagonal is the target \metricabbr{}.}
    \label{fig:minkpp-fdr-power-std}
\end{figure}

\subsection{Mixed-family audit pool}
\label{app:mixed-family}

The main-text protocol audits $K$ fine-tuned models initialized from a single backbone, which leaves open whether \abbr{}'s control depends on the audited models sharing an architecture and parameter scale. We complement that protocol with a heterogeneous audit pool composed of $K=16$ fine-tuned variants drawn from $8$ public LLM configurations spanning three families and a wide scale range: Pythia~\citep{biderman2023pythia} at $1.4$B, $2.8$B, $6.9$B, and $12$B, GPT-NeoX-$20$B~\citep{black-etal-2022-gpt}, and LLaMA~\citep{touvron2023llama} at $7$B, $13$B, and $30$B. Each of the $8$ configurations contributes two independently fine-tuned models over disjoint contamination assignments on the remainder block $\cP_{\mathrm{rem}}$, so the audit pool varies architecture, parameter scale, and per-model training data jointly. The shared-member calibration block $\cP_{\mathrm{cal}}$ from \cref{app:data-split} is reused by every variant unchanged; the rest of the protocol --- test split, four detection scores, and the target grid $\alpha\in\{0.1,0.2,0.3,0.4,0.5\}$ --- matches \cref{sec:setup}.

\cref{fig:mixed-family-multiscore} reports realized \fdpabbr{} and Power on this mixed-family pool for the four detection scores at $K=16$, on WikiMIA and ArXivTection. The realized \fdpabbr{} stays at or below the target diagonal throughout; Min-K\%++ delivers the highest Power on ArXivTection while the four scores cluster more tightly on WikiMIA. The picture matches \cref{sec:experiments}: \abbr{} controls \metricabbr{} without score-specific tuning and without architecture- or scale-matched audited models.

\begin{figure}[h]
    \centering
    \includegraphics[width=\textwidth]{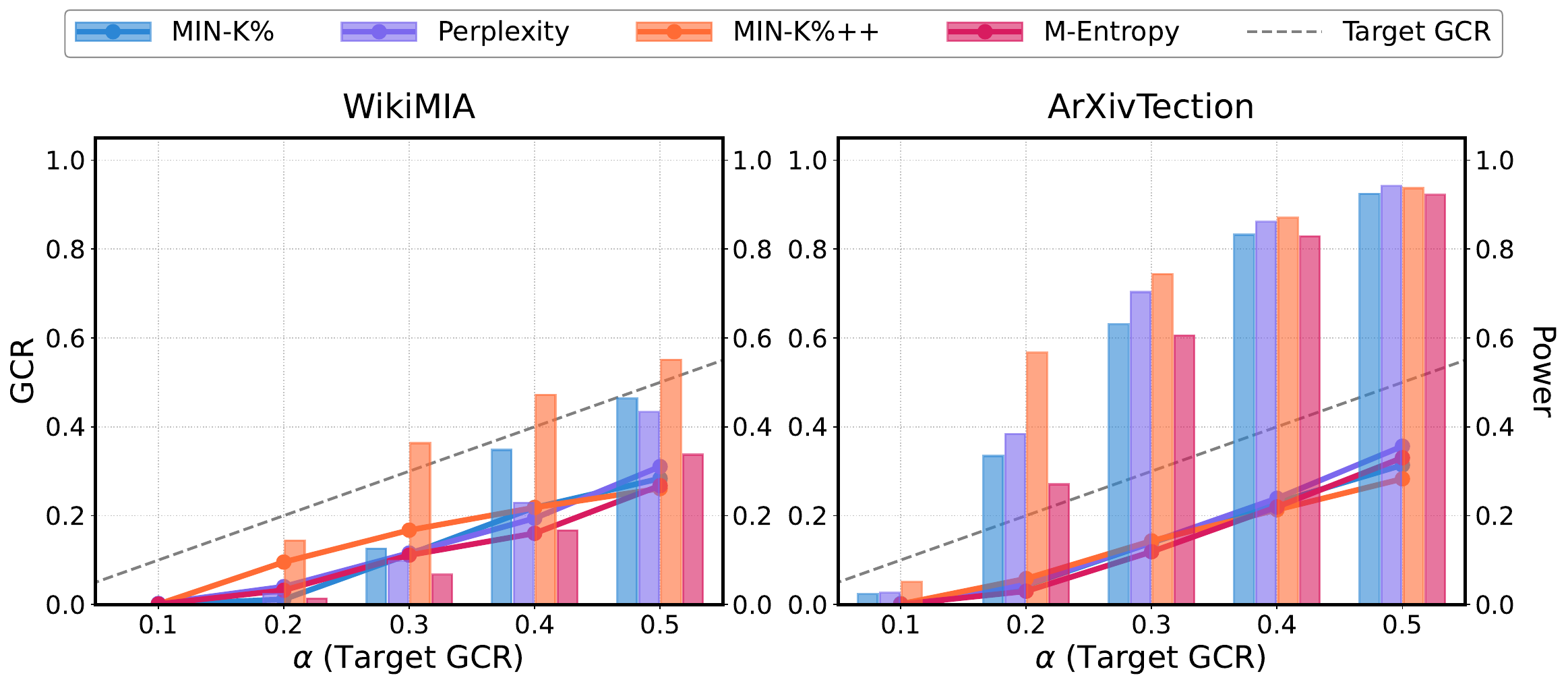}
    \caption{\abbr{} on a mixed-family audit pool at $K=16$. Realized \fdpabbr{} (lines, left axis) and Power (bars, right axis) for four detection scores at $\alpha\in\{0.1,0.2,0.3,0.4,0.5\}$ on WikiMIA (left) and ArXivTection (right). The $K=16$ pool combines $8$ fine-tuned configurations spanning Pythia ($1.4$B/$2.8$B/$6.9$B/$12$B), GPT-NeoX-$20$B, and LLaMA ($7$B/$13$B/$30$B), with two independently fine-tuned models per configuration. The dashed line is the target \metricabbr{}.}
    \label{fig:mixed-family-multiscore}
\end{figure}

\subsection{Recent open-weight model}
\label{app:recent-model}

To evaluate \abbr{} on a recently released model, we audit $K=16$ LoRA adapters fine-tuned from Qwen3.5-9B-Base, using the checkpoint revision published on March 2, 2026. We construct a post-release benchmark from a fixed snapshot of the Hugging Face arXiv metadata, selecting $800$ deduplicated abstracts of $32$ to $256$ words whose first version was submitted between May 1 and June 30, 2026. These abstracts postdate the checkpoint, so the base model is unlikely to have seen them during pretraining and membership is governed by our controlled fine-tuning. Following the three-block split of \Cref{app:data-split}, we use $240$ clean items, $240$ shared calibration members, and a remainder of $320$ items from which each adapter samples $40$, which corresponds to $\rho=0.125$. Each adapter uses the LoRA rank, scaling, and dropout of \Cref{app:finetune} with learning rate $10^{-3}$ and batch size $8$. Because Qwen3.5 interleaves full and linear attention layers, the LoRA targets also include the linear-attention projections. Detection scores are extracted from the epoch-$3$ checkpoint, and evaluation follows the protocol of \Cref{app:data-split}.

\Cref{tab:qwen35} shows that the empirical \metricabbr{} of \abbr{} stays below $\alpha$ for all four detection scores and target levels. Power exceeds $0.98$ throughout because the fine-tuned members are strongly separable from post-release abstracts under this model.

\begin{table}[h]
\centering
\caption{Empirical \metricabbr{} and Power of \abbr{} on post-release arXiv abstracts with $K=16$ Qwen3.5-9B-Base adapters.}
\label{tab:qwen35}
\vspace{4pt}
\begin{tabular}{l cc cc cc}
\toprule
\multirow{2}{*}{\textbf{Score}} & \multicolumn{2}{c}{$\alpha=0.1$} & \multicolumn{2}{c}{$\alpha=0.2$} & \multicolumn{2}{c}{$\alpha=0.3$} \\
\cmidrule(lr){2-3} \cmidrule(lr){4-5} \cmidrule(lr){6-7}
& \metricabbr{} & Power & \metricabbr{} & Power & \metricabbr{} & Power \\
\midrule
Perplexity & 0.061 & 0.989 & 0.143 & 0.996 & 0.227 & 0.996 \\
Min-K\%++ & 0.054 & 1.000 & 0.160 & 1.000 & 0.241 & 1.000 \\
Min-K\% & 0.058 & 0.995 & 0.140 & 0.997 & 0.229 & 0.997 \\
M-Entropy & 0.064 & 0.984 & 0.142 & 0.995 & 0.224 & 0.996 \\
\bottomrule
\end{tabular}
\end{table}

\section{Feasibility of Constructing the Shared-Member Calibration Set}
\label{app:calibration-feasibility}

The calibration assumption in \cref{sec:scaffold} requires a pool of items that are shared members of every audited model. We argue this is achievable in practice on three independent grounds.

\begin{enumerate}
    \item \textbf{Public benchmarks expose training splits that are routinely mixed into LLM training.} Public benchmarks release explicit training splits, and the contamination literature documents that such training data is in many cases deliberately incorporated into LLM pretraining or fine-tuning. For instance, OpenAI openly acknowledges mixing the MATH and GSM8K training sets into GPT-4's training~\citep{achiam2023gpt}, and surveys catalog the same practice across model families including GLM-130B, Qwen, Nemotron-4, InternLM-2, and MiniCPM~\citep{zhou2023don,xu2024benchmark}. Benchmark training splits thus constitute a sizable population of items that are members of many deployed models by construction.

    \item \textbf{Model providers openly disclose canonical web sources as documented pretraining components.} Modern LLM families often identify canonical web sources as components of their pretraining mixes, including Common Crawl, Wikipedia, arXiv, and GitHub. Llama publishes the full breakdown (CommonCrawl, C4, GitHub, Wikipedia, arXiv, StackExchange) in its technical report~\citep{touvron2023llama}, and Falcon similarly releases the RefinedWeb extraction of Common Crawl~\citep{penedo2023refinedweb}. Decontamination audits further confirm substantial overlap between such corpora and widely used benchmarks~\citep{yang2023rethinking}, so samples drawn from these disclosed sources can be treated as shared members of the corresponding model families.

    \item \textbf{Regulatory mandates are making training-data disclosure a legal requirement.} Beyond voluntary disclosure, Article~53(1)(d) of the EU AI Act (Regulation (EU) 2024/1689) requires providers of general-purpose AI (GPAI) models placed on the EU market, including OpenAI, Google, Anthropic, and others, to publish a ``sufficiently detailed summary'' of the content used for training, following a template from the EU AI Office~\citep{council2024regulation}. These obligations apply to newly placed GPAI models from 2~August~2025. Enforcement powers, including fines of up to EUR~15\,M or 3\% of global annual turnover, start on 2~August~2026, and pre-existing GPAI models must comply by 2~August~2027. As these summaries become available, auditors will have more high-confidence evidence about which public sources and data collections were used across major deployed models, making the shared-member assumption increasingly easy to satisfy.
\end{enumerate}

Each of the three grounds above provides realistic candidates for assembling the calibration pool in deployed audit settings. After collecting candidate shared members from these sources, the auditor restricts them to the benchmark domain and matches score-relevant properties such as topic, programming language, and sequence length before conformal calibration.

\end{document}